\documentclass[notitlepage]{article}

\usepackage{arxiv2}

\usepackage[ruled,vlined,linesnumbered]{algorithm2e}
\usepackage{amsthm}
\usepackage{amsfonts}       
\usepackage{amssymb}
\usepackage{booktabs}       
\usepackage{bm}
\usepackage{enumitem}
\usepackage[T1]{fontenc}    
\usepackage{hyperref}       
\usepackage[utf8]{inputenc} 
\usepackage{url}            
\usepackage{makecell}
\usepackage{mathtools}
\usepackage{multirow}
\usepackage{nicefrac}       
\usepackage{microtype}      
\usepackage[textsize=tiny]{todonotes}
\usepackage{thmtools}
\usepackage{xcolor}         

\SetKwInput{KwInput}{Input}
\SetKwInput{KwInit}{Initialize}
\SetKwInput{KwTest}{Test}
\SetKwInput{KwDef}{Definition}
\SetKwInput{KwReturn}{Return}

\newtheorem{theorem}{Theorem}
\newtheorem{lemma}[theorem]{Lemma}

\newtheorem{corollary[theorem]}{Corollary}
\newtheorem{definition}[theorem]{Definition}

\newtheorem{remark}{Remark}
\newtheorem{assumption}{Assumption}
\let\oldremark\remark
\renewcommand{\remark}{\oldremark\normalfont}

\usepackage[
  backend=biber,
  style=numeric-comp,
  maxbibnames=9,
  maxcitenames=1,
  uniquelist=false,
  uniquename=false,
  isbn=false,
  sorting=nty,
  doi=false
]{biblatex}
\DeclareMathOperator*{\argmin}{argmin}

\newcommand{\red}[1]{\textcolor{red}{#1}}

\newcount\Comments  
\newcommand{\kibitz}[2]{\ifnum\Comments=1\textcolor{#1}{#2}\fi}

\title{A Primal-Dual Algorithm for Offline Constrained Reinforcement Learning with Linear MDPs}
\author{%
  Kihyuk Hong \\
  University of Michigan \\
  \texttt{kihyukh@umich.edu} \\
  \And
  Ambuj Tewari \\
  University of Michigan \\
  \texttt{tewaria@umich.edu} \\
}

\begin{document}

\maketitle

\begin{abstract}
We study offline reinforcement learning (RL) with linear MDPs under the infinite-horizon discounted setting which aims to learn a policy that maximizes the expected discounted cumulative reward using a pre-collected dataset.
Existing algorithms for this setting either require a uniform data coverage assumptions or are computationally inefficient for finding an $\epsilon$-optimal policy with $\mathcal{O}(\epsilon^{-2})$ sample complexity.
In this paper, we propose a primal dual algorithm for offline RL with linear MDPs in the infinite-horizon discounted setting.
Our algorithm is the first computationally efficient algorithm in this setting that achieves sample complexity of $\mathcal{O}(\epsilon^{-2})$ with partial data coverage assumption.
Our work is an improvement upon a recent work that requires $\mathcal{O}(\epsilon^{-4})$ samples.
Moreover, we extend our algorithm to work in the offline constrained RL setting that enforces constraints on additional reward signals.
\end{abstract}

\section{Introduction}

We study the offline constrained reinforcement learning (RL) setting where a dataset of trajectories collected previously is given and the goal is to learn a decision making policy that performs well with respect to a reward signal while satisfying constraints on additional reward signals.
The setting is applicable to real-world problems that have safety concerns.
Learning from a previously collected dataset without interacting with the environment, a key property of offline RL \parencite{levine2020offline}, is useful in real-world problems where interacting with the environment is expensive or dangerous \parencite{kumar2021workflow,tang2021model,levine2018learning}.
Enforcing constraints on additional reward signals, a key property of constrained RL \parencite{altman2021constrained}, is useful for applications with safety concerns \parencite{wang2019safe,brunke2022safe}.

A challenge in offline RL is \textit{distribution shift} \parencite{levine2020offline}, a mismatch of the state-action distribution in the offline dataset to the state-action distributions induced by candidate policies.
For sample-efficient learning, offline RL requires the data distribution of the target policy to be covered by the offline dataset \parencite{chen2019information}.
A uniform data coverage assumption \parencite{antos2007fitted} is a convenient, but a strong assumption that requires the offline dataset to cover state-action distributions induced by all policies.
Recent works study offline RL with partial data coverage assumption that only requires the offline dataset to cover state-action distribution induced by a single target policy \parencite{jin2021pessimism}.

Another challenge in offline RL, which is also a challenge in online RL, is that many practical problems have large state spaces, making sample efficient learning difficult.
For sample efficient learning in large state space, we need to assume a structure in the problem.
In this paper, we study the linear MDP setting \parencite{jin2020provably} that assumes the transition probability matrix and the reward function have linear structures.
This setting ensures the value function is linear in a low-dimensional representation of state-action pairs, allowing sample-efficient learning.
To the best of our knowledge, none of the previous works on offline RL for linear MDPs provides with partial data coverage provide a computationally efficient algorithm with $\mathcal{O}(\epsilon^{-2})$ sample complexity.
In this paper, we introduce a novel algorithm that achieves this.
Furthermore, we extend to the offline \textit{constrained} RL setting that allows specifying constraints on additional reward signals.

\begin{table*}[t]
\caption{Comparison of algorithms for offline (constrained) RL}
\label{table:comparison}
\centering
\begin{tabular}{cccccc}
 \toprule
 \makecell{Setting} & Algorithm & \makecell{Partial \\ coverage} & \makecell{Computationally \\ efficient} & \makecell{Support \\ constraints} & N \\
 \midrule
 General & FQI \parencite{munos2008finite} & \red{No} & Yes & \red{No} & $\epsilon^{-2}$ \\
 General & CBPL \parencite{le2019batch} & \red{No} & Yes & Yes & $\epsilon^{-2}$ \\
 General & Minimax \parencite{xie2021bellman} & Yes & \red{No} & \red{No} & $\epsilon^{-2}$ \\
 General & CPPO \parencite{uehara2022pessimistic} & Yes & \red{No} & \red{No} & $\epsilon^{-2}$ \\
 General & Minimax \parencite{zanette2023realizability} & \red{No} & \red{No} & \red{No} & $\epsilon^{-2}$ \\
 \midrule
 Linear & PSPI \parencite{xie2021bellman} & Yes & Yes & \red{No} & $\red{\epsilon^{-5}}$ \\
 Linear MDPs & Primal-Dual \parencite{gabbianelli2023offline} & Yes & Yes & \red{No} & $\red{\epsilon^{-4}}$ \\
 Linear MDPs & Primal-Dual (\textbf{Ours}) & Yes & Yes & Yes & $\epsilon^{-2}$ \\
 \bottomrule
\end{tabular}
\end{table*}

\subsection{Related Work}

In Table~\ref{table:comparison}, we compare our work to previous works.
The column $N$ shows how the sample complexity bound scales with the error tolerance $\epsilon$.
The first five algorithms are for offline RL with general function approximation. The algorithms can be reduced to the linear function approximation setting by taking a value function class consisting of linear functions.
The sixth algorithm is for offline RL with linear function approximation.
The last two algorithms are results on offline RL with linear MDPs, which is a special case of the linear function approximation setting.
The computational efficiency of algorithms for the general function approximation setting is judged based on the efficiency when applied to linear function class.
As the table shows, our algorithm is the first computationally efficient algorithm with sample complexity $\mathcal{O}(\epsilon^{-2})$ for finding $\epsilon$-optimal policy under partial data coverage assumption.
Moreover, our algorithm supports constraints on additional reward signals.

\paragraph{Offline RL with General Function Approximation}

Offline RL with general function approximation is widely studied in the discounted infinite-horizon setting.
When casting the linear function approximation setting to the general function approximation setting, we get the realizability and Bellman completeness for free when using linear function class since the value function under linear function approximation is linear.
In Table~\ref{table:comparison}, we only compared works on general function approximation that assumes realizability, Bellman completeness and data coverage.
There are other works that relax Bellman completeness assumption at the cost of introducing another assumption.
For example, \textcite{xie2020q,zhan2022offline,zhu2023importance,hong2023primal} relax Bellman completeness assumption and introduce marginalized importance weight assumption.

\paragraph{Offline RL with Episodic Setting}

Offline RL with linear function approximation has been studied in the finite-horizon episodic setting.
\textcite{zanette2021provable} propose a computationally efficient actor-critic algorithm with pessimism to achieve $\mathcal{O}(\epsilon^{-2})$ sample complexity under partial data coverage. \textcite{jin2021pessimism} propose a computationally efficient value iteration based algorithm with pessimism to achieve $\mathcal{O}(\epsilon^{-2})$ sample complexity under partial data coverage. However, they require the knowledge of the covariance matrix induced by the state-action data distribution.
Although their results are computationally efficient and work under partial data coverage, they do not apply to the infinite-horizon discounted setting.
\textcite{wu2021offline} study offline constrained RL with a more general way of specifying constraints. Their focus is on episodic setting with linear mixture MDP.

\section{Preliminaries}

\paragraph{Notations}
We denote by $\Delta(\mathcal{X})$ the probability simplex over a finite set $\mathcal{X}$.
We write $\bm\Delta^I = \{ \bm{x} \in \mathbb{R}_+^I : \sum_{i = 1}^I x_i \leq 1 \}$.
We write $\mathbb{B}_d(B) = \{ \bm{x} \in \mathbb{R}^d : \Vert \bm{x} \Vert_2 \leq B \}$.
Given a matrix $A$, denote by $A^\dagger$ its pseudoinverse.

We consider a Markov decision process (MDP) $\mathcal{M} = (\mathcal{S}, \mathcal{A}, P, r, \gamma, \nu_0)$ where $\mathcal{S}$ is the state space, $\mathcal{A}$ is the action space, $P : \mathcal{S} \times \mathcal{A} \rightarrow \Delta(\mathcal{S})$ is the probability transition kernel, $r : \mathcal{S} \times \mathcal{A} \rightarrow [0, 1]$, is the reward function, $\gamma$ is the discount factor and $\nu_0$ is the initial state distribution.
We assume that initial state is fixed to $s_0$ for simplicity.
We assume $\mathcal{S}$ and $\mathcal{A}$ are finite, but potentially very large.
We assume the reward function $r$ is deterministic and known to the learner.
The probability transition kernel $P$ is unknown to the learner. 

The interaction protocol between the learner and the MDP is as follows.
The learner interacts with the MDP starting from the initial state $s_0 \in \mathcal{S}$.
At each step $t = 0, 1, \dots$, the learner chooses an action $a_t \in \mathcal{A}$ and observes the reward $r(s_t, a_t)$ and the next state $s_{t + 1}$.
The next state $s_{t + 1}$ is drawn by the environment from $P(\cdot | s_t, a_t)$.

We define the normalized expected cumulative rewards
$$
J(\pi) \coloneqq (1 - \gamma) \mathbb{E}^\pi \left[\sum_{t = 0}^\infty \gamma^t r(s_t, a_t) \right]
$$
where $\mathbb{E}^\pi$ is the expectation with respect to the distribution of the trajectory $(s_0, a_0, s_1, a_1, \dots)$ induced by the interaction of the probability transition $P$ and the policy $\pi$.
The normalizing factor $1 - \gamma$ makes $J(\pi) \in [0, 1]$ for all $i = 0, \dots, I$.

The goal of RL is to find a policy $\pi : \mathcal{S} \rightarrow \Delta(\mathcal{A})$ that maximizes the reward.

\subsection{Linear MDP}

For sample-efficient learning for arbitrarily large state space, we assume access to a feature mapping $\bm\varphi : \mathcal{S} \times \mathcal{A} \rightarrow \mathbb{R}^d$ that reduces the dimension of the problem as follows.

\begin{assumption}[Linear MDP] \label{assumption:linear-unconstrained}
We assume that the transition and the reward functions can be expressed as a linear function of a \textit{known} feature map $\bm\varphi : \mathcal{S} \times \mathcal{A} \rightarrow \mathbb{R}^d$ such that
$$
r(s, a) = \langle \bm\varphi(s, a), \bm\theta \rangle,\quad
P(s' | s, a) = \langle \bm\varphi(s, a), \bm\psi(s') \rangle
$$
for all $(s, a, s') \in \mathcal{S} \times \mathcal{A} \times \mathcal{S}$ where $\bm\theta \in \mathbb{R}^d$ is a \textit{known} parameter for the reward function and $\bm\psi = (\psi_1, \dots, \psi_d)$ is a vector of $d$ \textit{unknown} (signed) measures on $\mathcal{S}$.
\end{assumption}
The linear MDP assumption is widely studied in the RL literature for studying theoretical properties of RL with function approximation \parencite{jin2020provably}.
As is commonly done in works on linear MDPs, we further make the following boundedness assumptions.
Without loss of generality (see Appendix A in \textcite{wei2021learning} for justification), we assume
$$
\Vert \bm\varphi(s, a) \Vert_2 \leq 1,\quad \Vert \bm\theta \Vert_2 \leq \sqrt{d}
$$ for all $(s, a) \in \mathcal{S} \times \mathcal{A}$.
We further make a technical assumption, also made by \textcite{wagenmaker2022first}, that for some constant $D_\psi$,
$$
\Vert \vert \bm\psi \vert( \mathcal{S} ) \Vert_2 \leq D_\psi \sqrt{d}
$$
where $\vert \bm\psi \vert(\mathcal{S}) = \sum_{s \in \mathcal{S}} (\vert \psi_1(s) \vert, \dots, \vert \psi_d(s)  \vert)$.
This assumption holds, for example, when $\psi_i$ are probability measures on $\mathcal{S}$, in which case the assumption holds with $D_\psi = 1$.

The linear structure implies a low-dimensional factorization of key quantities as we discuss below.
Let $\bm{P} \in \mathbb{R}^{\vert \mathcal{S} \times \mathcal{A} \vert \times \vert \mathcal{S} \vert}$ be the matrix representation of the probability transition kernel $P$ with $(\bm{P})_{(s, a), s'} = P(s' | s, a)$ for $(s, a, s') \in \mathcal{S} \times \mathcal{A} \times \mathcal{S}$.
Then, the linear structure gives
$$
\bm{P} = \bm\Phi \bm\Psi
$$
where $\bm\Psi \in \mathbb{R}^{d \times \vert \mathcal{S} \vert}$ is the \textit{unknown} matrix of all measures with rows $(\psi_i(s'))_{s' \in \mathcal{S}}$ and $\bm\Phi \in \mathbb{R}^{\vert \mathcal{S} \times \mathcal{A} \vert \times d}$ is the \textit{known} matrix of all feature vectors with rows $(\bm\varphi(s, a))_{(s, a) \in \mathcal{S} \times \mathcal{A}}$.

Let $Q^\pi$ be the action value function of a policy $\pi$ with respect to a reward function $r: \mathcal{S} \times \mathcal{A} \rightarrow \mathbb{R}$ is defined as follows.
$$
Q^\pi(s, a) = \mathbb{E}^\pi \left[ \sum_{t = 0}^\infty \gamma^t r(s_t, a_t) | s_0 = s, a_0 = a \right].
$$
It is the expected discounted cumulative reward starting from the state-action pair $(s, a)$ and then executing the policy $\pi$ every time step.
Similarly, the state value function of a policy $\pi$ with respect to a reward function $r$ is defined as
$$
V^\pi(s) = \mathbb{E}^\pi \left[ \sum_{t = 0}^\infty \gamma^t r(s_t, a_t) | s_0 = s\right].
$$
It is the expected discounted cumulative reward starting from the state $s$ and then executing the policy $\pi$ every time step.

We use $\bm{Q}^\pi \in \mathbb{R}^{\vert \mathcal{S} \times \mathcal{A} \vert}$ to denote the matrix representation of the function $Q^\pi$ such that $(\bm{Q}^\pi)_{s, a} = Q^\pi(s, a)$ and $\bm{V}^\pi \in \mathbb{R}^{\vert \mathcal{S} \vert}$ the vector representation of the function $V^\pi$ such that $\bm{V}^\pi_s = V^\pi(s)$.
With these notations, the well known Bellman equation $Q^\pi(s, a) = r(s, a) + \gamma PV^\pi(s, a)$
(see e.g. \textcite{puterman2014markov}) can be written as
\begin{equation} \label{eqn:bellman}
\bm{Q}^\pi = \bm{r} + \gamma \bm{P} \bm{V}^\pi = \bm\Phi(\bm\theta + \gamma \bm\Psi \bm{V}^\pi) = \bm\Phi \bm\zeta^\pi
\end{equation}
where $\bm{r} \in \mathbb{R}^{\vert \mathcal{S} \times \mathcal{A} \vert}$ is the matrix representation of the reward function $r$ and we define $\bm\zeta^\pi \coloneqq \bm\theta + \gamma \bm\Psi \bm{V}^\pi \in \mathbb{R}^d$.
This shows that the action value function is linear in the feature vector:
$$
Q^\pi(s, a) = \langle \bm\varphi(s, a), \bm\zeta^\pi \rangle.
$$
Due to the boundedness assumptions $\Vert \bm\theta \Vert_2 \leq \sqrt{d}$ and $\Vert \vert \bm\psi \vert ( \mathcal{S} ) \Vert_2 \leq D_\psi \sqrt{d}$, and the fact that $V^\pi(s) \in [0, \frac{1}{1 - \gamma}]$, the norm of the parameter $\bm\zeta^\pi$ is bounded by
$$
\Vert \bm\zeta^\pi \Vert_2 \leq \sqrt{d} + \frac{\gamma D_\psi \sqrt{d}}{1 - \gamma} = \mathcal{O}\left(\frac{D_\psi \sqrt{d}}{1 - \gamma}\right).
$$
We define $D_\zeta \coloneqq \sqrt{d} + \frac{\gamma D_\psi \sqrt{d}}{1 - \gamma}$.

\subsection{Offline Learning and Data Coverage}

We consider the offline learning setting where the agent has access to a dataset $\mathcal{D} = (s_j, a_j, s_j')_{j = 1}^n$.
The pairs $(s_j, a_j)$, $j = 1, \dots, n$ are assumed to be i.i.d. samples from a distribution $\mu_B \in \Delta(\mathcal{S} \times \mathcal{A})$ and each $s_j'$ is sampled from $P(\cdot~|~s_j, a_j)$.
Such an i.i.d. assumption on the offline dataset is commonly made in the offline RL literature \cite{xie2021bellman,zhan2022offline,chen2022offline,zhu2023importance} to facilitate the analysis of concentration bounds.

A major challenge in sample efficient offline RL is \textit{distribution shift}, which refers to the mismatch of state-action distribution of the offline dataset and the target (optimal) policy.
For sample efficient learning, we require an assumption on data coverage that guarantees the distribution of target policy is covered by the offline dataset.
A common data coverage assumption in offline RL is concentrability assumption that limits the ratio of occupancy measure of target policy to that of behavior policy.
The normalized occupancy measure of a policy $\pi$ is defined as
$$
\mu^\pi(s, a) = (1 - \gamma) \mathbb{E}^\pi \left[ \sum_{t = 0}^\infty \gamma^t \mathbb{I}\{ s_t = s, a_t = a \} \right].
$$
Roughly, it is the normalized count of the visitation of state-action pair $(s, a)$ when executing the policy $\pi$.
It is normalized to ensure $\mu^\pi(s, a)$ is a probability measure on $\mathcal{S} \times \mathcal{A}$ so that $\sum_{s, a} \mu(s, a) = 1$.
We use the notation $\bm\mu^\pi \in \mathbb{R}^{\vert \mathcal{S} \times \mathcal{A} \vert}$ to denote the matrix representation of the function $\mu^\pi(s, a)$.
A commonly used concentrability assumption is as follows.
\begin{assumption}[Concentrability] \label{assumption:concentrability}
For an optimal policy $\pi^\ast$, we have
$$
\frac{\mu^\ast(s, a)}{\mu_B(s, a)} \leq C^\ast
$$
for all $s \in \mathcal{S}$ and $a \in \mathcal{A}$ with $\mu_B(s, a) > 0$ where we write $\mu^\ast = \mu^{\pi^\ast}$.
The bound $C^\ast$ is known to the learner.
\end{assumption}
Concentrability assumption is widely used in offline RL with general function approximation \parencite{munos2003error,munos2005error}.
The assumption requires the ratio $\mu^\ast(s, a) / \mu_B(s, a)$ to be bounded for all $(s, a) \in \mathcal{S} \times \mathcal{A}$ for which $\mu_B(s, a)$ is positive.
As discussed by \textcite{gabbianelli2023offline}, in the linear function approximation setting with access to a feature map, we can define data coverage in the feature space rather than in the state-action space.
We defer discussion on our result that uses data coverage assumption in feature space to Section~\ref{section:feature-coverage}.

\section{Algorithm Design}

Our algorithm is motivated by the linear programming formulation of the reinforcement learning problem.
$$
\begin{aligned}
\max_{\bm\mu \geq \bm{0}} \quad\quad &\langle \bm{r}, \bm{\mu} \rangle \\
\text{subject to} \quad\quad
&\bm{E}^T \bm{\mu} = (1 - \gamma) \bm{\nu}_0 + \gamma \bm{P}^T \bm\mu.
\end{aligned}
$$
Here, $\bm{E} \in \mathbb{R}^{\vert \mathcal{S} \times \mathcal{A} \vert \times \vert \mathcal{S} \vert}$ denotes the matrix with $\bm{E}_{(s, a), s'} = \mathbb{I} \{s = s' \}$ and $\bm\nu_0$ denotes the initial state distribution, which is assumed to be $\bm{e}_{s_0}$.
Note that the $s$th entry of $\bm{E}^T \bm\mu$ is $\sum_a \mu(s, a)$, which is the sum of $\mu(s, \cdot)$ over all possible values of $a$.
The optimization variable $\bm\mu \in \mathbb{R}^{\vert \mathcal{S} \times \mathcal{A} \vert}$ has the interpretation of the normalized occupancy measure.
The objective function has the interpretation of the value of a policy, which can be seen by
$$
J(\pi) = (1 - \gamma) \mathbb{E}^\pi \left[ \sum_{t = 0}^\infty \gamma^t r(s_t, a_t) \right] = \langle \bm{r}, \bm\mu^\pi \rangle.
$$
The constraint $\bm{E}^T \bm\mu = (1 - \gamma) \bm\nu_0 + \gamma \bm{P}^T \bm\mu$, called Bellman flow constraint, makes sure that $\bm\mu$ is a permissible occupancy measure in the sense that there exists a policy $\pi$ that induces the measure, i.e., $\mu = \mu^\pi$ for some policy $\pi$.

We use $\bm{r} = \bm\Phi \bm\theta$ and $\bm{P} = \bm\Phi \bm\Psi$, which hold by the linear MDP assumption (Assumption~\ref{assumption:linear-unconstrained}), to rewrite the linear program as
$$
\begin{aligned}
\max_{\bm\mu \geq 0} \quad\quad &\langle \bm\theta, \bm\Phi^T \bm{\mu} \rangle \\
\text{subject to} \quad\quad
&\bm{E}^T \bm{\mu} = (1 - \gamma) \bm{\nu}_0 + \gamma \bm\Psi^T \bm\Phi^T \bm\mu
\end{aligned}
$$
Note that the optimization variable $\bm\mu \in \mathbb{R}^{\vert \mathcal{S} \times \mathcal{A} \vert}$ is high-dimensional that depends on the size of $\mathcal{S}$.
Following \textcite{gabbianelli2023offline}, with the goal of computational and statistical efficiency, we introduce a low-dimensional optimization variable $\bm\lambda = \bm\Phi^T \bm\mu \in \mathbb{R}^d$, which has the interpretation of the average occupancy in the feature space.
With the reparametrization, the optimization problem becomes
$$
\begin{aligned}
\max_{\bm\mu \geq \bm{0}, \bm\lambda} \quad\quad &\langle \bm{\theta}, \bm{\lambda} \rangle \\
\text{subject to} \quad\quad
&\bm{E}^T \bm{\mu} = (1 - \gamma) \bm{\nu}_0 + \gamma \bm\Psi^T \bm\lambda \\
&\bm\lambda = \bm\Phi^T \bm\mu.
\end{aligned}
$$
The dual of the linear program above is
$$
\begin{aligned}
\min_{\bm{v}, \bm\zeta} \quad\quad
& (1 - \gamma) \langle \bm\nu_0, \bm{v} \rangle \\
\text{subject to} \quad\quad
& \bm\zeta = \bm\theta + \gamma \bm\Psi \bm{v} \\
& \bm{E} \bm{v} \geq \bm\Phi \bm\zeta.
\end{aligned}
$$
The dual variable $\bm{v} \in \mathbb{R}^{\vert \mathcal{S} \vert}$ has the interpretation of the vector representation of the state value function and $\bm\zeta \in \mathbb{R}^d$ the parameter such that $\bm\Phi \bm\zeta \in \mathbb{R}^{\vert \mathcal{S} \times \mathcal{A} \vert}$ is the vector representation of the state-action value function.
The Lagrangian associated to this pair of linear programs is
\begin{align*}
&L(\bm\lambda, \bm\mu; \bm{v}, \bm\zeta) \\
&= (1 - \gamma)\langle \bm\nu_0, \bm{v} \rangle + \langle \bm\lambda, \bm\theta + \gamma \bm\Psi \bm{v} - \bm\zeta \rangle
+ \langle \bm\mu, \bm\Phi \bm\zeta - \bm{E} \bm{v} \rangle \\
&= \langle \bm\lambda, \bm\theta \rangle + \langle \bm{v}, (1 - \gamma) \bm\nu_0 + \gamma \bm\Psi^T \bm\lambda - \bm{E}^T \bm\mu \rangle \\
&\quad\quad+ \langle \bm\zeta, \bm\Phi^T \bm\mu - \bm\lambda \rangle.
\end{align*}

Note that the optimization variables $\bm\lambda, \bm\zeta \in \mathbb{R}^d$ are low-dimensional, but $\bm\mu \in \mathbb{R}^{\vert \mathcal{S} \times \mathcal{A} \vert}$ and $\bm{v} \in \mathbb{R}^{\vert \mathcal{S} \vert}$ are not.
With the goal of running a primal-dual algorithm on the Lagrangian using only low-dimensional variables, we introduce policy variable $\pi$ and parameterize $\bm\mu$ and $\bm{v}$, following \textcite{gabbianelli2023offline}, by
\begin{align}
\mu_{\bm\lambda, \pi}(s, a) &= \pi(a | s) \left[
  (1 - \gamma) \nu_0(s) + \gamma \langle \psi(s), \bm\lambda \rangle
\right] \label{eqn:mu} \\
v_{\bm\zeta, \pi}(s) &= \sum_a \pi(a | s) \langle \bm\zeta, \bm\varphi(s, a) \rangle. \label{eqn:v}
\end{align}
The choice of $\bm\mu_{\bm\lambda, \pi}$ makes the Bellman flow constraint $\bm{E}^T \bm\mu_{\bm\lambda, \pi} = (1 - \gamma) \bm\nu_0 + \gamma \bm\Psi^T \bm\lambda$ of the primal problem satisfied.
Also, the choice of $\bm{v}_{\bm\zeta, \pi}$ makes $\langle \bm\mu, \bm\Phi \bm\zeta - \bm{E} \bm{v}_{\bm\zeta, \pi} \rangle = 0$.
Using the above parameterization, the Lagrangian can be rewritten in terms of $\bm\zeta, \bm\lambda, \pi$ as follows:
\begin{align}
f(&\bm\lambda, \bm\zeta, \pi) \notag \\
&= \langle \bm\lambda, \bm\theta_0 \rangle + \langle \bm\zeta, \bm\Phi^T \bm\mu_{\bm\lambda, \pi} - \bm\lambda \rangle \label{eqn:lagrangian1-unconstrained} \\
&= (1 - \gamma) \langle \bm\nu_0, \bm{v}_{\bm\zeta, \pi} \rangle + \langle \bm\lambda, \bm\theta_0 + \gamma \bm\Psi \bm{v}_{\bm\zeta, \pi} - \bm\zeta \rangle. \label{eqn:lagrangian2-unconstrained}
\end{align}
At the cost of having to keep track of $\pi$, we can now run a primal-dual algorithm on the low-dimensional variables $\bm\zeta$, $\bm\lambda$.
The introduction of $\pi$ in the equation does not make the algorithm inefficient because we can only keep track of the distribution $\pi(s | a)$ for state-action pairs that appear in the dataset. See Appendix~\ref{appendix:computation} for detail.

Previous work on offline linear MDP \parencite{gabbianelli2023offline} runs primal-dual algorithm on the variables $\bm\zeta$ and $\bm\beta = \bm\Lambda^\dagger \bm\lambda$ by estimating the gradient of the Lagrangian with respect to the variables.
Their algorithm requires running gradient descent algorithm on $\bm\zeta$ for every gradient ascent step of $\bm\beta$, leading to a double-loop algorithm structure.
Since each gradient descent/ascent step requires fresh copy of independent data, the double-loop algorithm leads to sample complexity of $\mathcal{O}(\epsilon^{-4})$.
We sidestep the need of the double-loop structure and obtain $\mathcal{O}(\epsilon^{-2})$ sample complexity by restricting the values of $\bm\lambda$ to a carefully designed confidence set that allows estimating the gradient uniformly over the choices of $\bm\lambda$, $\bm\zeta$ and $\pi$.
We outline the argument in the following section.

\subsection{Analysis}

For a given policy $\pi$, recall that $\bm\zeta^\pi \in \mathbb{R}^d$ is the parameter satisfying $\bm{Q}^\pi = \bm\Phi \bm\zeta^\pi$.
It can be shown that for any $\bm\lambda \in \mathbb{R}^d$,
$$
f(\bm\zeta^\pi, \bm\lambda, \pi) = J(\pi).
$$
Also, defining $\bm\lambda^\pi = \bm\Phi^T \bm\mu^\pi$, which has the interpretation of the average occupancy in the feature space when executing the policy $\pi$, it can be shown that for any $\bm\zeta \in \mathbb{R}^d$,
$$
f(\bm\zeta, \bm\lambda^\pi, \pi) = J(\pi).
$$
See Appendix~\ref{appendix:lagrangian} for proofs.
Hence, for any sequences $\{ \pi_t \}$, $\{ \bm\theta_t \} \subset \mathbb{R}^d$ and $\{ \bm\lambda_t \} \subset \mathbb{R}^d$, we have
\begin{align*}
J&(\pi^\ast) - J(\pi_t) = f(\bm\zeta_t, \bm\lambda^\ast, \pi^\ast) - f(\bm\zeta^{\pi_t}, \bm\lambda_t, \pi_t) \\
&=
(\underbrace{f(\bm\zeta_t, \bm\lambda^\ast, \pi^\ast) - f(\bm\zeta_t, \bm\lambda^\ast, \pi_t)}_{\textsc{Reg}_t^\pi}) \\
&\hspace{10mm}+
(\underbrace{f(\bm\zeta_t, \bm\lambda^\ast, \pi_t) - f(\bm\zeta_t, \bm\lambda_t, \pi_t)}_{\textsc{Reg}_t^\lambda}) \\
&\hspace{10mm}+
(\underbrace{f(\bm\zeta_t, \bm\lambda_t, \pi_t) - f(\bm\zeta^{\pi_t}, \bm\lambda_t, \pi_t)}_{\textsc{Reg}_t^\zeta})
\end{align*}
where we use the notation $\bm\lambda^\ast = \bm\lambda^{\pi^\ast}$.

Note that the suboptimality $J(\pi^\ast) - J(\pi_t)$ is decomposed into regrets of the three players.
As long as we show that the sums of the three regrets over $t = 1, \dots, T$ are sublinear in $T$, we obtain $\frac{1}{T} \sum_{t = 1}^T J(\pi^\ast) - J(\pi_t) = J(\pi^\ast) - J(\bar{\pi}) = o(1)$ where $\bar{\pi} = \text{Unif}(\pi_1, \dots, \pi_T)$ is the mixture policy that chooses a policy among $\pi_1, \dots, \pi_T$ uniformly at random and runs the chosen policy for the entire trajectory.

In the rest of the section, we sketch analyses of bounding the regrets of the three players.
These analyses will motivate our algorithm presented in Section~\ref{section:algorithm}.

\subsection{Bounding Regret of $\pi$-player}

Using Equation \eqref{eqn:lagrangian2-unconstrained}, the regret of $\pi$-player simplifies to
\begin{align*}
\text{Reg}^\pi_t &=
f(\bm\zeta_t, \bm\lambda^\ast, \pi^\ast) - f(\bm\zeta_t, \bm\lambda^\ast, \pi_t) \\
&= \langle \bm\nu^\ast, \bm{v}_{\bm\zeta_t, \pi^\ast} - \bm{v}_{\bm\zeta_t, \pi_t} \rangle \\
&= \langle \bm\nu^\ast, \textstyle \sum_a (\pi^\ast(a | \cdot ) - \pi_t(a | \cdot)) \langle \bm\zeta_t, \bm\varphi(\cdot, a) \rangle \rangle.
\end{align*}
where we define $\bm\nu^\pi = (1 - \gamma) \bm\nu_0 + \gamma \bm\Psi^T \bm\lambda^\pi$ as the state occupancy measure induced by $\pi$ and write $\bm\nu^\ast = \bm\nu^{\pi^\ast}$.
The regret can be bounded if $\pi$-player updates its policy using an exponentiation algorithm \parencite{zanette2021provable}
$$
\pi_{t + 1} = \sigma\left(\alpha \sum_{i = 1}^t \bm\Phi \bm\zeta_i\right)
$$
where $\sigma(\bm{q})$ for $\bm{q} \in \mathbb{R}^{\vert \mathcal{S} \times \mathcal{A} \vert}$ is a softmax policy with
$$
\sigma(\bm{q})(a | s) \coloneqq \frac{\exp(q(s, a))}{\sum_{a'} \exp(q(s, a'))}.
$$
Based on the standard mirror descent analysis by \textcite{gabbianelli2023offline} (Appendix~\ref{appendix:pi-player}) we can show that, choosing $\alpha = \mathcal{O}((1 - \gamma)\sqrt{\log \vert \mathcal{A} \vert / (dT)})$ gives
$$
\frac{1}{T} \sum_{t = 1}^T \textsc{Reg}_t^\pi \leq \mathcal{O}\left(\frac{1}{1 - \gamma}\sqrt{(d \log \vert \mathcal{A} \vert) / T }\right)
$$
which vanishes as $T$ increases.
Consequently, choosing $T$ to be at least $\Omega(\frac{d \log \vert \mathcal{A} \vert}{(1 - \gamma)^2 \epsilon^2})$ gives $\frac{1}{T} \sum_{t = 1}^T \text{Reg}_t^\pi \leq \epsilon$.
Note that when the exponentiation algorithm is employed, the $\pi$-player does not need to know the value of $\bm\zeta_t$ when choosing $\pi_t$, allowing the $\pi$-player to play before the $\zeta$-player.
Another benefit of the exponentiation algorithm is that the policy chosen by the $\pi$-player is restricted to the softmax function class $\Pi(D_\pi)$ where $\Pi(\cdot)$ is defined as
\begin{equation} \label{eqn:softmax-class}
\Pi(B) \coloneqq \{ \sigma(\bm\Phi\bm{z}) : \bm{z} \in \mathbb{B}_d(B) \}.
\end{equation}
and $D_\pi \coloneqq \alpha T D_\zeta$.
The restriction allows statistically efficient estimation of quantities that depend on policies in $\Pi(B)$ via covering argument on $\Pi(B)$, as we will see in later sections.

\subsection{Bounding Regret of $\zeta$-player}
Using Equation \eqref{eqn:lagrangian1-unconstrained}, the regret of $\zeta$-player simplifies to
\begin{align*}
\textsc{Reg}_t^\zeta &= f(\bm\zeta_t, \bm\lambda_t, \pi_t) - f(\bm\zeta^{\pi_t}, \bm\lambda_t, \pi_t) \\
&=
\langle \bm\zeta_t - \bm\zeta^{\pi_t}, \bm\Phi^T \bm\mu_{\bm\lambda_t, \pi_t} - \bm\lambda_t \rangle.
\end{align*}

Recall that $\bm\mu_{\bm\lambda, \pi} = \pi \circ \bm{E}[(1 - \gamma) \bm\nu_0 + \gamma \bm\Psi^T \bm\lambda]$.
The only unknown quantity in the regret is $\bm\Psi^T \bm\lambda \in \mathbb{R}^{\vert \mathcal{S} \vert}$.
Note that $\bm\Psi^T \bm\varphi(s, a) = (P(s' | s, a))_{s' \in \mathcal{S}}$ is a next-state distribution given current state-action pair $(s, a)$, and $\bm{e}_{s'_k}$ is an unbiased estimator for $\bm\Psi^T \bm\varphi(s_k, a_k)$.
Hence, if $\bm\lambda$ is a linear combination $\sum_{k = 1}^n c_k \bm\varphi(s_k, a_k)$, we can construct an unbiased estimator $\sum_{k = 1}^n c_k \bm{e}_{s_k'}$ for $\bm\Psi^T \bm\lambda$.
Motivated by this observation, to facilitate the algorithm design for the $\zeta$-player, we will restrict the $\lambda$-player to choose a linear combination of feature vectors that appear in the dataset to allow estimating $\bm\Psi^T \bm\lambda$.
Specifically, we strict the $\lambda$-player to choose $\bm\lambda_t$ from the following set where the bound $B$ will be chosen later.
\begin{equation} \label{eqn:confidence-set}
\mathcal{C}_n(B) \coloneqq \left\{ \frac{1}{n} \sum_{k = 1}^n c_k \bm\varphi(s_k, a_k) : c_1, \dots, c_n \in [-B, B] \right\}.
\end{equation}
Given the restriction, we can parameterize the value of $\bm\lambda_t$ by the coefficients $\bm{c}_t \in [-B, B]^n$ for some bound $B$, and write $\bm\lambda_t = \bm\lambda(\bm{c}_t)$ where we define
$$
\bm\lambda(\bm{c}) \coloneqq \frac{1}{n} \sum_{k = 1}^n c_k \bm\varphi(s_k, a_k)
$$
Following the previous discussion, we define the estimates for $\bm\Psi^T \bm\lambda(c)$ and $\bm\mu_{\bm\lambda(\bm{c}), \pi}$ parameterized by $\bm{c} \in [-B, B]^n$:
\begin{align*}
\widehat{\bm\Psi^T \bm\lambda}(\bm{c}) &\coloneqq \frac{1}{n} \sum_{k = 1}^n c_k \bm{e}_{s_k'} \\
\widehat{\bm\mu}_{\bm\lambda(\bm{c}), \pi} &\coloneqq \pi \circ \bm{E}[(1 - \gamma) \bm\nu_0 + \gamma \widehat{\bm\Psi^T \bm\lambda}(\bm{c})].
\end{align*}
These estimates enjoy the following concentration bound, which can be shown using matrix Bernstein inequality.
See Appendix~\ref{appendix:zeta-regret} for a proof.
\begin{lemma} \label{lemma:phiTmu}
For a fixed $\bm\lambda(\bm{c}) = \frac{1}{n} \sum_{k = 1}^n c_k \bm\varphi(s_k, a_k)$ with $\vert c_k \vert \leq B$ for $k = 1, \dots, n$, and a policy $\pi$, we have
$$
\Vert \bm\Phi^T \bm\mu_{\bm\lambda(\bm{c}), \pi} - \bm\Phi^T \widehat{\bm\mu}_{\bm\lambda(\bm{c}), \pi} \Vert_2 \leq \mathcal{O}\left(B \sqrt{\frac{\log(d / \delta)}{n}}\right)
$$
with probability at least $1 - \delta$ conditional on the data of state-action pairs $\{ (s_k, a_k) \}_{k = 1}^n$.
\end{lemma}

For estimating the regret $\langle \bm\zeta_t - \bm\zeta^{\pi_t}, \bm\Phi^T \bm\mu_{\bm\lambda(\bm{c}_t), \pi_t} - \bm\lambda(\bm{c}_t) \rangle$, we need a uniform concentration bound on the estimates $\bm\Phi^T \widehat{\bm\mu}_{\bm\lambda(\bm{c}_t), \pi}$ over $\bm\lambda(\bm{c})$ and $\pi$.
The restriction on the $\pi$-player to choose a policy in the softmax function class defined in \eqref{eqn:softmax-class} allows converting the concentration bound for a fixed policy $\pi$ in Lemma~\ref{lemma:phiTmu} to a uniform concentration bound over all policies in the softmax function class via a covering argument.
The conversion is possible due to the fact that the log covering number for the softmax function class is bounded by $\widetilde{\mathcal{O}}(d)$ (see Lemma~\ref{lemma:cover-pi} in Appendix~\ref{appendix:covering}).
However, such a conversion to a uniform concentration bound over all $\bm\lambda(\bm{c})$ for $\bm{c} \in [-B, B]^n$ is elusive since a naive covering argument on the space of parameters $[-B, B]^n$ will give a log covering number bound of $\mathcal{O}(n)$.
To sidestep this issue, we exploit the fact that $\mathcal{C}_n(B)$ can be spanned by a set of spanners $\{ \bm\varphi(s_j, a_j) \}_{j \in \mathcal{I}}$ for some index set $\mathcal{I} \subseteq \{ 1, \dots, n \}$ of size at most $d$.
This can be seen by the following lemma by \textcite{awerbuch2008online}.
\begin{lemma}[Barycentric spanner]
Let $\mathcal{K} \subseteq \mathbb{R}^d$ be compact set. 
Then, there exists a spanner $\{ \bm\phi_1, \dots \bm\phi_d \} \subset \mathcal{K}$ such that any vector $\bm{x} \in \mathcal{K}$ can be represented as $\bm{x} = \sum_{i = 1}^d c_i \bm\phi_i$ where $c_i \in [-C, C]$ for all $i = 1, \dots, d$.
Such a spanner is called a $C$-approximate barycentric spanner for $\mathcal{K}$.
If $\mathcal{K}$ is finite, we can find a $C$-approximate barycentric spanner in time complexity $\mathcal{O}(nd^2 \log_C d)$.
\end{lemma}
Applying this lemma, we can compute a $2$-approximate barycentric spanner $\{ \bm\varphi(s_j, a_j) \}_{j \in \mathcal{I}}$ for $\{ \bm\varphi(s_k, a_k) \}_{k = 1}^n$ where $\mathcal{I} \subseteq \{ 1, \dots, n\}$ is an index set of size $d$.
Given any $\bm{c} \in [-B, B]^n$, we can convert it to $\bm{c}' \in [-2B, 2B]^n$ with $c_k'$ nonzero only if $k \in \mathcal{I}$ such that $\bm\lambda(\bm{c}) = \bm\lambda(\bm{c}')$.
This can be seen by
\begin{align*}
\bm\lambda(\bm{c}) &= \frac{1}{n} \sum_{k = 1}^n c_k \bm\varphi(s_k, a_k) \notag \\
&= \sum_{j \in \mathcal{I}} \left(\frac{1}{n} \sum_{k = 1}^n b_{kj} c_k \right) \bm\varphi(s_j, a_j) = \bm\lambda(\bm{c}')
\end{align*}
where the coefficients $b_{kj} \in [-2, 2]$ are such that $\bm\varphi(s_k, a_k) = \sum_{j \in \mathcal{I}} b_{kj} \bm\varphi(s_j, a_j)$, which exist by the fact that $\{ \bm\varphi(s_j, a_j) \}_{j \in \mathcal{I}}$ is a 2-approximate barycentric spanner of $\{ \bm\varphi(s_k, a_k) \}_{k = 1}^n$.
We summarize the definition of the conversion from $\bm{c}$ to $\bm{c}'$ that satisfies $\bm\lambda(\bm{c}) = \bm\lambda(\bm{c}')$.
\begin{definition} \label{definition:conversion}
Given a dataset $\{ (s_k, a_k, s_k') \}_{k = 1}^n$, let $\{ \bm\varphi(s_j, a_j) \}_{j \in \mathcal{I}}$ be a 2-approximate barycentric spanner for $\{ \bm\varphi(s_k, a_k) \}_{k = 1}^n$ with $\vert \mathcal{I} \vert \leq d$.
We define the conversion of $\bm{c} \in \mathbb{R}^n$ to $\bm{c}' \in \mathbb{R}^n$ as
\begin{equation}
c_j' = \begin{cases}
\frac{1}{n} \sum_{k = 1}^n b_{kj} c_k & \text{if}~ j \in \mathcal{I} \\
0 & \text{otherwise}.
\end{cases}
\end{equation}
where $b_{kj}$ are the coefficients such that $\bm\varphi(s_k, a_k) = \sum_{j \in \mathcal{I}} b_{kj} \bm\varphi(s_j, a_j)$ with $b_{kj} = 0$ for $j \notin \mathcal{I}$.
\end{definition}

Given $\bm{c}_t \in [-B, B]^n$ such that $\bm\lambda(\bm{c}_t) \in \mathcal{C}_n(B)$, let $\bm{c}_t' \in [-2B, 2B]^n$ be the conversion such that $\bm\lambda(\bm{c}_t') = \bm\lambda(\bm{c}_t)$ with only the coefficients with indices in $\mathcal{I}$ nonzero.
The converted coefficients $\bm{c}_t' \in \mathbb{R}^n$ live in a low dimensional space $\{ \bm{c}' \in [-2B, 2B]^n : c_j' = 0 ~\text{if} ~j \notin \mathcal{I} \}$ with the log covering number of $\mathcal{O}(d)$.
To use a covering argument, let $\bm{c}_t''$ be the covering center closest to $\bm{c}_t'$.
Then we can decompose the regret of the $\zeta$-player as
\begin{align*}
\textsc{Reg}_t^\zeta
&= \langle \bm\zeta_t - \bm\zeta^{\pi_t}, \bm\Phi^T \bm\mu_{\bm\lambda(\bm{c}_t), \pi_t} - \bm\lambda(\bm{c}_t) \rangle \\
&= \langle \bm\zeta_t - \bm\zeta^{\pi_t}, \bm\Phi^T \bm\mu_{\bm\lambda(\bm{c}_t'), \pi_t} - \bm\Phi^T \bm\mu_{\bm\lambda_t(\bm{c}''_t), \pi_t} \rangle \\
&\quad\quad+ \langle \bm\zeta_t - \bm\zeta^{\pi_t}, \bm\Phi^T \bm\mu_{\bm\lambda(\bm{c}''_t), \pi_t} - \bm\Phi^T \widehat{\bm\mu}_{\bm\lambda(\bm{c}''_t), \pi_t} \rangle \\
&\quad\quad+ \langle \bm\zeta_t - \bm\zeta^{\pi_t}, \bm\Phi^T \widehat{\bm\mu}_{\bm\lambda(\bm{c}_t''), \pi_t} - \bm\Phi^T \widehat{\bm\mu}_{\bm\lambda(\bm{c}_t'), \pi_t} \rangle \\
&\quad\quad+ \langle \bm\zeta_t - \bm\zeta^{\pi_t}, \bm\Phi^T \widehat{\bm\mu}_{\bm\lambda(\bm{c}_t'), \pi_t} - \bm\lambda(\bm{c}_t') \rangle.
\end{align*}
The first term can be bounded since $\bm\lambda(c_t') \approx \bm\lambda_t(\bm{c}_t'')$.
The second term can be bounded using a union bound of the concentration inequalities on $\bm\Phi^T \widehat{\bm\mu}_{\bm\lambda(\bm{c}''), \pi}$ over $\bm{c}''$ in the cover of $[-2B, 2B]^d$.
The third term can be bounded since $c_t' \approx c_t''$.
The last term, interpreted as a regret of the $\zeta$-player against a dynamic action $\bm{\zeta}^{\pi_t}$, can be bounded by a greedy $\zeta$-player that minimizes $\langle \cdot, \bm\Phi^T \widehat{\bm\mu}_{\bm\lambda(\bm{c}_t'), \pi_t} - \bm\lambda(\bm{c}_t') \rangle$.
The greedy strategy requires $\zeta$-player to play after $\lambda$-player and $\pi$-player.
The bounds lead to
$$
\frac{1}{T}\sum_{t = 1}^T \textsc{Reg}_t^\zeta \leq
\mathcal{O}\left(\frac{B d}{1 - \gamma} \sqrt{\frac{\log (B dn / \delta)}{n}}\right).
$$
In summary, we can bound the regret of $\zeta$-player if $\zeta$-player plays $\bm\zeta_t \in \mathbb{B}_d(D_\zeta)$ that minimizes $\langle \cdot, \bm\Phi^T \widehat{\bm\mu}_{\bm\lambda(\bm{c}_t'), \pi_t} - \bm\lambda(\bm{c}_t') \rangle$.
The greedy strategy requires $\zeta$-player to play after $\lambda$-player and $\pi$-player.
See Appendix~\ref{appendix:zeta-regret} for detailed analysis.

\subsection{Bounding Regret of $\lambda$-player}
Using Equation \eqref{eqn:lagrangian2-unconstrained}, the regret of $\lambda$-player simplifies to
\begin{align*}
\textsc{Reg}_t^\lambda &= f(\bm\zeta_t, \bm\lambda^\ast, \pi_t) - f(\bm\zeta_t, \bm\lambda_t, \pi_t) \\
&=
\langle \bm\lambda^\ast - \bm\lambda_t, \underbrace{\bm\theta + \gamma \bm\Psi \bm{v}_{\bm\zeta_t, \pi_t} - \bm\zeta_t}_{= \bm\xi_t} \rangle
\end{align*}
The sum of $\textsc{Reg}_t^\lambda$ over $t = 1, \dots, T$ is the regret of the $\lambda$-player against a fixed action $\lambda^\ast$ where the reward function at time $t$ is $\langle \cdot, \bm\xi_t \rangle$.
From the previous section, we require $\lambda$ player to play before $\zeta$-player, whose play affects $\bm\xi_t$.
Hence, the decision of $\lambda$-player at time $t$ must be made before the knowledge of $\bm\xi_t$.
Assuming for now that $\bm\xi_t$ is known ($\bm\xi_t$ is in fact unknown and needs to be estimated since $\bm\Psi$ is unknown), the regret of $\lambda$-player can be made sublinear in $T$ by employing a no-regret online convex optimization oracle (defined below) on $\mathcal{C}_n(B)$ as long as $\bm\lambda^\ast \in \mathcal{C}_n(B)$.
\begin{definition}\label{def:oco}
An algorithm is called a \textit{no-regret online convex optimization oracle} with respect to a convex set $\mathcal{C}$ if, for
any sequence of convex functions $h_1, \dots, h_T : \mathbb{R}^d \rightarrow [-1, 1]$ and for any $\bm\lambda \in \mathcal{C}$, the sequence of vectors $\bm\lambda_1, \dots, \bm\lambda_T \in \mathcal{C}$ produced by the algorithm satisfies
$$
\frac{1}{T} \sum_{t = 1}^T h_t(\bm\lambda_t) - h_t(\bm\lambda) \leq \epsilon_{\text{opt}}^\lambda(T)
$$
for some $\epsilon_{\text{opt}}^\lambda(T) > 0$ that converges to 0 as $T \rightarrow \infty$.
\end{definition}
The online gradient descent algorithm \parencite{hazan2016introduction} is an example of a computationally efficient online convex optimization oracle.
Employing a no-regret online convex optimization oracle with convex set $\mathcal{C}_n(B)$ on the sequence of functions $\langle \cdot, \bm\xi_t \rangle$, the $\lambda$-player can enjoy a sublinear regret against any fixed $\bm\lambda \in \mathcal{C}_n(B)$.
However, $\bm\lambda^\ast$ may not lie in $\mathcal{C}_n(B)$ for any $B$.
In fact, we can construct an example where $\bm\lambda^\ast$ is not in the span of $\{ \bm\varphi(s_k, a_k) \}_{k = 1}^n$ with probability at least $1/2$ (Lemma~\ref{lemma:impossible}).
To sidestep this problem, we show $\bm\lambda^\ast$ can be approximated by a vector $\widehat{\bm\lambda}^\ast$ in $\mathcal{C}_n(C^\ast)$:
\begin{lemma} \label{lemma:lambda-hat}
Under the concentrability assumption~\ref{assumption:concentrability}, there exists $\widehat{\bm\lambda}^\ast \in \mathbb{R}^d$ of the form $\widehat{\bm\lambda}^\ast = \frac{1}{n} \sum_{k = 1}^n c_k \bm\varphi(s_k, a_k)$ with $c_k \in [0, C^\ast]$, $k = 1, \dots, n$ such that
$$
\Vert \widehat{\bm\lambda}^\ast - \bm\lambda^\ast \Vert_2 \leq \mathcal{O}\left(C^\ast \sqrt{\textstyle \frac{\log(d / \delta)}{n}} \right)
$$
with probability at least $1 - \delta$.
\end{lemma}
Note that this lemma is the only place the data coverage assumption is needed for our analysis and this is where we require choosing $B = C^\ast$.
Also, in our algorithm, computing $\widehat{\bm\lambda}^\ast$ is not needed.
Only the existence of such a vector $\widehat{\bm\lambda}^\ast$ is needed in the analysis.
With the lemma above, we can approximate the regret as
\begin{align*}
\textsc{Reg}_t^\lambda &=
\langle \bm\lambda^\ast - \bm\lambda_t, \bm\theta + \gamma \bm\Psi \bm{v}_{\bm\zeta_t, \pi_t} - \bm\zeta_t \rangle \\
&\approx
\langle \widehat{\bm\lambda}^\ast - \bm\lambda_t, \bm\theta + \gamma \bm\Psi \bm{v}_{\bm\zeta_t, \pi_t} - \bm\zeta_t \rangle
\end{align*}
and argue that the sum of the above quantity over $t = 1, \dots, T$ is sublinear since the regret against $\widehat{\bm\lambda}^\ast \in \mathcal{C}_n(C^\ast)$ is sublinear when employing a no-regret online convex optimization oracle.
Now, we deal with the fact that the term $\bm\Psi \bm{v}_{\bm\zeta_t, \pi_t}$ is unknown and needs to be estimated.
Observing that $\langle \bm\varphi(s, a), \bm\Psi \bm{v} \rangle = \mathbb{E}_{s' \sim P(\cdot | s, a)}[\bm{v}(s')]$, we can estimate $\bm\Psi \bm{v} \in \mathbb{R}^d$ for any $\bm{v} \in \mathbb{R}^{\vert \mathcal{S} \vert}$ by regressing $v(s')$ on $\bm\varphi(s, a)$ using the triplets $(s, a, s')$ in the dataset $\mathcal{D}$.
Following the literature on linear bandits \parencite{abbasi2011improved}, we use the regularized least squares estimate
\begin{equation} \label{eqn:least-squares}
\widehat{\bm\Psi \bm{v}} \coloneqq (n\widehat{\bm\Lambda}_n + \bm{I})^{-1} \textstyle \sum_{k = 1}^n v(s_k') \bm\varphi(s_k, a_k)
\end{equation}
where $\widehat{\bm\Lambda}_n \coloneqq \frac{1}{n} \sum_{k = 1}^n \bm\varphi(s_k, a_k) \bm\varphi(s_k, a_k)^T$ is the empirical Gram matrix.
By the well-known result for linear bandits (e.g. Theorem 2 in \textcite{abbasi2011improved}), we have the following high-probability concentration bound (Lemma~\ref{lemma:ls}) for the estimate $\widehat{\bm\Psi \bm{v}}$ where $v : \mathcal{S} \rightarrow [0, D_v]$:
$$
\Vert \bm\Psi \bm{v} - \widehat{\bm\Psi \bm{v}} \Vert_{n \widehat{\bm\Lambda}_n + \bm{I}}
\leq \mathcal{O}\left(
D_v \sqrt{d \log (n / \delta)}
\right).
$$
Since we need concentration bound of $\widehat{\bm\Psi \bm{v}}_{\bm\zeta_t, \pi_t}$ where $\bm{v}_{\bm\zeta_t, \pi_t}$ are random, we need a uniform bound over all possible functions $v_{\bm\zeta_t, \pi_t}$.
Since the domain of $v$ has cardinality $\vert \mathcal{S} \vert$, a naive covering argument on the function space of $v$ will make the bound scale with $\text{poly}(\vert \mathcal{S} \vert)$.
To avoid this, we use a careful covering argument exploiting the fact that $\pi_t$ is a softmax function parameterized by a $d$-dimensional vector and $\bm\zeta_t$ are $d$-dimensional vectors.
With covering, we can show the following uniform concentration bound.
\begin{lemma} \label{lemma:ls-uniform}
Consider a function class
$$
\mathcal{V} = \left\{ v_{\bm\zeta, \pi} \in (\mathcal{S} \rightarrow [0, \textstyle D_v]) : \bm\zeta \in \mathbb{B}(D_\zeta), \pi \in \Pi(D_\pi) \right\}.
$$
With probability at least $1 - \delta$, we have
$$
\Vert \bm\Psi \bm{v} - \widehat{\bm\Psi \bm{v}} \Vert_{n \widehat{\bm\Lambda}_n + \bm{I}}
\leq \mathcal{O}\left(
D_v \sqrt{d\log(D_\zeta D_\pi n / \delta)}
\right)
$$
uniformly over $\bm{v} \in \mathcal{V}$ where $\widehat{\bm\Psi \bm{v}}$ is the least squares estimate defined in \eqref{eqn:least-squares}.
\end{lemma}
See Lemma~\ref{lemma:ls-uniform} in the Appendix for detail.
With the uniform concentration bound, we can continue bounding the regret of $\lambda$-player as follows.
\begin{align*}
&\langle \widehat{\bm\lambda}^\ast - \bm\lambda_t, \bm\theta + \gamma \bm\Psi \bm{v}_{\bm\zeta_t, \pi_t} - \bm\zeta_t \rangle \\
&= \langle \widehat{\bm\lambda}^\ast - \bm\lambda_t, \bm\theta + \gamma \widehat{\bm\Psi \bm{v}}_{\bm\zeta_t, \pi_t} - \bm\zeta_t \rangle \\
&\quad\quad+ \gamma \langle \widehat{\bm\lambda}^\ast - \bm\lambda_t, \bm\Psi \bm{v}_{\bm\zeta_t, \pi_t} - \widehat{\bm\Psi \bm{v}}_{\bm\zeta_t, \pi_t} \rangle.
\end{align*}
The sum of the first term across $t = 1, \dots, T$ can be bounded by employing online convex optimization algorithm.
We can bound the second term as follows.
\begin{align*}
\langle &\widehat{\bm\lambda}^\ast - \bm\lambda_t, \bm\Psi \bm{v}_{\bm\zeta_t, \pi_t} - \widehat{\bm\Psi \bm{v}}_{\bm\zeta_t, \pi_t} \rangle \\
&\leq
\Vert \widehat{\bm\lambda}^\ast - \bm\lambda_t \Vert_{(\widehat{\bm\Lambda}_n + \bm{I} / n)^{-1}} \Vert \bm\Psi \bm{v}_{\bm\zeta_t, \pi_t} - \widehat{\bm\Psi \bm{v}}_{\bm\zeta_t, \pi_t} \Vert_{\widehat{\bm\Lambda}_n + \bm{I} / n} \\
&\leq
\underbrace{\Vert \widehat{\bm\lambda}^\ast - \bm\lambda_t \Vert_{\widehat{\bm\Lambda}_n^\dagger}}_{(i)} \underbrace{\Vert \bm\Psi \bm{v}_{\bm\zeta_t, \pi_t} - \widehat{\bm\Psi \bm{v}}_{\bm\zeta_t, \pi_t} \Vert_{\widehat{\bm\Lambda}_n + \bm{I} / n}}_{(ii)}
\end{align*}
where the second inequality follows since $\widehat{\bm\lambda}^\ast - \bm\lambda_t$ is in the column space of $\widehat{\bm\Lambda}_n$.
$(ii)$ can be bounded using Lemma~\ref{lemma:ls-uniform}.
$(i)$ can be bounded by the following technical lemma.
See Appendix~\ref{appendix:confidence1} for a proof.

\begin{lemma} \label{lemma:confidence1}
For any $\bm\lambda(\bm{c}) = \frac{1}{n} \sum_{k = 1}^n c_k \bm\varphi(s_k, a_k)$ with $c_k \in [-B, B]$, $k = 1, \dots, n$, we have
$$
\Vert \bm\lambda(\bm{c}) \Vert_{\widehat{\bm\Lambda}_n^\dagger}^2 \leq d B^2.
$$
\end{lemma}
Combining all the bounds, we get the following.
$$
\frac{1}{T} \sum_{t = 1}^T \textsc{Reg}^\lambda_t \leq \widetilde{\mathcal{O}}\left(\frac{C^\ast d^{3/ 2}}{1 - \gamma} \sqrt{\frac{\log(dnT / \delta)}{n}} \right) + \epsilon_{\text{opt}}^\lambda(T)
$$
where $\widetilde{\mathcal{O}}$ hides $\log\log \vert \mathcal{A} \vert$.
See Appendix~\ref{appendix:regret-lambda} for details.

\section{Algorithm and Main Results} \label{section:algorithm}

\begin{algorithm}
\KwInput{Dataset $\mathcal{D} = \{(s_j, a_j, r_j, s_j') \}_{j = 1}^n$.}
\KwInit{$\pi_1$ uniform, $\bm{c}_1' \leftarrow \bm{0}$, $\alpha \leftarrow \sqrt{\log \vert \mathcal{A} \vert / T}$.}
\For{$t = 1, \dots, T$}{
  $\bm\zeta_t \leftarrow \argmin_{\bm\zeta \in \mathbb{B}_d(D_\zeta)} \langle \bm\zeta, \bm\Phi^T \widehat{\bm\mu}_{\bm\lambda(\bm{c}_t'), \pi_t} - \bm\lambda(\bm{c}_t') \rangle$. \label{alg-line:zeta-unconstrained} \\
  $\bm\lambda(\bm{c}_{t + 1}) \leftarrow \textsc{OCO}(\bm\theta - \bm\zeta_t + \gamma \widehat{\bm\Psi \bm{v}}_{\bm\zeta_t, \pi_t}; \mathcal{C}_n(C^\ast))$. \label{alg-line:lambda-unconstrained} \\
  Convert $\bm{c}_{t + 1}$ to $\bm{c}'_{t + 1}$ using Definition~\ref{definition:conversion}. \\
  $\pi_{t + 1} \leftarrow \sigma(\alpha \sum_{i = 1}^t \bm\Phi \bm\zeta_i)$.
}
\KwReturn{$\bar\pi = \text{Unif}(\pi_1, \dots, \pi_T)$}
\caption{Primal-Dual Algorithm for Offline Linear MDPs}
\label{alg:pdapc}
\end{algorithm}

Motivated by the analysis in the previous section for bounding regrets of the four players, we present a primal-dual algorithm that proceeds in $T$ steps.
At each step, the four players $\lambda$-player, $\zeta$-player, $w$-player, $\pi$-player choose actions $\bm\lambda_t$, $\bm\zeta_t$, $\bm{w}_t$, $\pi_t$, respectively.
Since the analysis in the previous section requires $\zeta$-player and $w$-player to act greedily, we choose $\lambda$-player and $\pi$-player to play $\bm\lambda_t$, $\pi_t$, respectively, before $\zeta$-player and $w$-player play.
The regret analysis of the three players in the previous section leads to our main result in the following theorem.

\begin{theorem} \label{thm:main-unconstrained}
Under Assumptions \ref{assumption:linear-unconstrained} and \ref{assumption:concentrability}, as long as $T$ is at least $\Omega(\frac{d \log \vert \mathcal{A} \vert}{(1 - \gamma)^2 \epsilon^2})$, the policy $\bar\pi$ produced by Algorithm~\ref{alg:pdapc} satisfies $J(\bar\pi) \geq J(\pi^\ast) - \epsilon$ with probability at least $1 - \delta$ for sample size 
$$
n = \mathcal{O}\left(
\frac{(C^\ast)^2 d^3 \log(dn (\log \vert \mathcal{A} \vert) / (\delta \epsilon(1 - \gamma)))}{(1 - \gamma)^2 \epsilon^2}
\right).
$$
\end{theorem}
Our work is an improvement over the work by \textcite{gabbianelli2023offline} who give $\widetilde{\mathcal{O}}(\frac{(C^\ast)^2 d^2 \log \vert \mathcal{A} \vert}{(1 - \gamma)^4 \epsilon^2})$ sample complexity.

\subsection{Result on Feature Coverage Assumptions} \label{section:feature-coverage}
The discussion so far uses the concentrability assumption (Assumption~\ref{assumption:concentrability}) that requires $\mu^\ast(s, a) / \mu(s, a) \leq C^\ast$ for all $s \in \mathcal{S}$ and $a \in \mathcal{A}$.
We present a result that requires a feature coverage assumption instead.
We use the definition of feature coverage used in \textcite{gabbianelli2023offline}:
\begin{assumption}[Feature coverage] \label{assumption:feature-coverage}
For an optimal policy $\pi^\ast$, we have 
$$
(\bm\lambda^\ast)^T (\bm\Lambda^\dagger)^2 \bm\lambda^\ast \leq C^\ast ~~\text{and}~~ \bm\lambda^\ast \in \text{Col}(\bm\Lambda)
$$
where $\bm\lambda^\ast \coloneqq \mathbb{E}_{\mu^\ast}[\bm\varphi(s, a)]$ and $\text{Col}(\cdot)$ is the column space.
\end{assumption}

The feature coverage assumption requires $\bm\lambda^\ast$, the expected occupancy of the target policy in the feature space, to be covered by covariance matrix induced by the data distribution $\mu_B$.
Under the feature coverage assumption, $\bm\lambda^\ast$ can be approximated by a linear combination of $\bm\varphi(s_k, a_k)$, $k = 1, \dots, n$ (Lemma~\ref{lemma:lambda-hat2}).
This result is analogous to Lemma~\ref{lemma:lambda-hat} that uses concentrability assumption instead.
It follows that the result in Theorem~\ref{thm:main-unconstrained} with the concentrability assumption (Assumption~\ref{assumption:concentrability}) replaced by the feature coverage assumption (Assumption~\ref{assumption:feature-coverage}).
A limitation of our work is that we use a stronger notion of feature coverage compared to the one used by \textcite{gabbianelli2023offline}, who assume $(\bm\lambda^\ast)^T (\bm\Lambda^\dagger) \bm\lambda^\ast$ is bounded. However, they require the knowledge of $\bm\Lambda$ and their sample complexity is $\widetilde{\mathcal{O}}(\epsilon^{-4})$.
We leave design and analysis of algorithm using the weaker notion of feature coverage to future work.

\section{Extension to Offline Constrained RL}

We now consider a \textit{constrained} Markov decision process (CMDP) $\mathcal{M} = (\mathcal{S}, \mathcal{A}, P, \{ r_i \}_{i = 0}^I, \gamma, \nu_0)$.
It is the same setting as the MDP setting except that now we have multiple reward functions $r_i$, $i = 0, \dots, I$.
We define the normalized expected cumulative rewards for $r_0, \dots, r_I$:
$$
J_i(\pi) \coloneqq (1 - \gamma) \mathbb{E}^\pi \left[\sum_{t = 0}^\infty \gamma^t r_i(s_t, a_t) \right].
$$

Constrained RL aims to find a policy $\pi : \mathcal{S} \rightarrow \Delta(\mathcal{A})$ that maximizes the reward signal $r_0$ subject to the constraints on other reward signals $r_i$, $i = 1, \dots, I$.
Specifically, given thresholds $\bm\tau = (\tau_1, \dots, \tau_I)$, the goal is to find $\pi$ that solves the following optimization problem denoted by $\mathcal{P}(\bm\tau)$.
\begin{equation} \label{eqn:opt} \tag{OPT}
\begin{aligned}
\max_{\pi} \quad &J_0(\pi) \\
\text{subject to} \quad &J_i(\pi) \geq \tau_i,~~ i = 1, \dots, I.
\end{aligned}
\end{equation}

We assume the following Slater's condition, a commonly made assumption in constrained RL \parencite{le2019batch,chen2021primal,bai2023achieving,ding2020natural} for ensuring strong duality of the optimization problem.
\begin{assumption}[Slater's condition] \label{assumption:slater}
There exist a constant $\phi > 0$ and a policy $\pi$ such that $J_i(\pi) \geq \tau_i + \phi$ for all $i = 1, \dots, I$.
Assume $\phi$ is known.
\end{assumption}
As discussed in \textcite{hong2023primal},
Slater's condition is a mild assumption since given the knowledge of the feasibility of the problem, we can guarantee that Slater's condition is met by slightly loosening the cost threshold.
For sample efficient learning for arbitrarily large state space, we assume the following linear structure on the CMDP.

\begin{assumption}[Linear CMDP] \label{assumption:linear-cmdp}
We assume that the transition and the reward functions can be expressed as a linear function of a \textit{known} feature map $\bm\varphi : \mathcal{S} \times \mathcal{A} \rightarrow \mathbb{R}^d$ such that
$$
r_i(s, a) = \langle \bm\varphi(s, a), \bm\theta_i \rangle,\quad
P(s' | s, a) = \langle \bm\varphi(s, a), \bm\psi(s') \rangle
$$
for all $(s, a, s') \in \mathcal{S} \times \mathcal{A} \times \mathcal{S}$ and $i = 1, \dots, I$, where $\bm\theta_i \in \mathbb{R}^d$ are \textit{known} parameters and $\bm\psi = (\psi_1, \dots, \psi_d)$ is a vector of $d$ \textit{unknown} (signed) measures on $\mathcal{S}$.
\end{assumption}
Similarly to the linear MDP setting, we require the data coverage assumption (Assumption~\ref{assumption:concentrability}) where the optimal policy $\pi^\ast$ is optimal for the optimization problem $\mathcal{P}(\bm\tau)$.

Our algorithm for the linear CMDP setting is motivated by the linear programming formulation of the constrained reinforcement learning problem \eqref{eqn:opt}:
$$
\begin{aligned}
\max_{\bm\mu \geq \bm{0}} \quad\quad &\langle \bm{r}_0, \bm{\mu} \rangle \\
\text{subject to} \quad\quad
& \langle \bm{r}_i, \bm\mu \rangle \geq \tau_i, \quad i = 1, \dots, I, \\
&\bm{E}^T \bm{\mu} = (1 - \gamma) \bm{\nu}_0 + \gamma \bm{P}^T \bm\mu.
\end{aligned}
$$
and its dual
$$
\begin{aligned}
\min_{\bm{w} \geq \bm{0}, \bm{v}, \bm\zeta} \quad\quad
& (1 - \gamma) \langle \bm\nu_0, \bm{v} \rangle - \langle \bm{w}, \bm\tau \rangle \\
\text{subject to} \quad\quad
& \bm\zeta = \bm\theta_0 + \bm\Theta \bm{w} +  \gamma \bm\Psi \bm{v} \\
& \bm{E} \bm{v} \geq \bm\Phi \bm\zeta.
\end{aligned}
$$
where we write $\bm\Theta = \begin{bmatrix} \bm\theta_1 & \cdots & \bm\theta_I \end{bmatrix} \in \mathbb{R}^{d \times I}$.

\begin{algorithm}
\KwInput{Dataset $\mathcal{D} = \{(s_j, a_j, r_j, s_j') \}_{j = 1}^n$, $D_w$, $\bm\tau$}
\KwInit{$\pi_1$ uniform, $\bm{c}'_1 \leftarrow \bm{0}$, $\alpha \leftarrow \sqrt{\log \vert \mathcal{A} \vert / T}$.}
\For{$t = 1, \dots, T$}{
  $\bm\zeta_t \leftarrow \argmin_{\bm\zeta \in \mathbb{B}_d(D_\zeta)} \langle \bm\zeta, \bm\Phi^T \widehat{\bm\mu}_{\bm\lambda(\bm{c}_t'), \pi_t} - \bm\lambda(\bm{c}_t') \rangle$. \label{alg-line:zeta-constrained} \\
  $\bm{w}_t \leftarrow \argmin_{\bm{w} \in D_w \bm\Delta^I} \langle \bm{w}, \bm\tau - \bm\Theta^T \bm\lambda_t \rangle$. \\
  $\bm\lambda(\bm{c}_{t + 1}) \leftarrow \textsc{OCO}(\bm\theta_0 - \bm\zeta_t + \bm\Theta \bm{w}_t + \gamma \widehat{\bm\Psi \bm{v}}_{\bm\zeta_t, \pi_t}; \mathcal{C}_n(C^\ast))$. \label{alg-line:lambda-constrained} \\
  Convert $\bm{c}_{t + 1}$ to $\bm{c}'_{t + 1}$ using Definition~\ref{definition:conversion}. \\
  $\pi_{t + 1} \leftarrow \sigma(\alpha \sum_{i = 1}^t \bm\Phi \bm\zeta_i)$.
}
\KwReturn{$\bar\pi = \text{Unif}(\pi_1, \dots, \pi_T)$}
\caption{Primal-Dual Algorithm for Offline Linear CMDPs}
\label{alg:pdapc-constrained}
\end{algorithm}

The structure of our algorithm for the constrained RL setting is similar to that for the unconstrained RL setting.
The difference is that we add the $w$-player that adjusts the weights on the rewards $r_1, \dots, r_I$.
Closely following the analysis for the unconstrained setting, we can show the following sample complexity for the constrained setting.

\begin{theorem} \label{theorem:main}
Under Assumptions~\ref{assumption:concentrability},\ref{assumption:slater} and \ref{assumption:linear-cmdp}, the policy $\bar\pi$ produced by Algorithm~\ref{alg:pdapc} with threshold $\bm\tau$ and $D_w = 1 + \frac{1}{\phi}$ and $T$ at least $\Omega(\frac{d \log \vert \mathcal{A} \vert}{(1 - \gamma)^2 \epsilon^2})$ and large enough such that $\epsilon^\lambda_\text{opt}(T) \leq \epsilon$ satisfies $J_0(\bar\pi) \geq J_0(\pi^\ast) - \epsilon$ and $J_i(\bar\pi) \geq \tau_i - \epsilon$ with probability at least $1 - \delta$ with sample size 
$$
n = \mathcal{O}\left(
\frac{(C^\ast)^2 d^3 \log(dn (\log \vert \mathcal{A} \vert) / (\delta \phi \epsilon (1 - \gamma)))}{(1 - \gamma)^2 \phi^2 \epsilon^2}
\right).
$$
\end{theorem}
See Appendix~\ref{appendix:constrained} for details.
By tightening the input thresholds for Algorithm~\ref{alg:pdapc} to $\bm\tau + \phi \epsilon \bm{1}$ and assuming a two-policy feature coverage assumption, we can show that the output policy $\bar\pi$ is $\epsilon$-optimal and satisfies the constraints exactly, i.e., $J_i(\bar\pi) \geq \tau_i$, $i = 1, \dots, I$.
See Appendix~\ref{appendix:exact-feasibility} for details.

\section{Conclusion}

In this paper, we propose a computationally efficient primal dual algorithm for offline constrained RL with linear function approximation under partial data coverage.
Our algorithm is the first computationally efficient algorithm to achieve $\mathcal{O}(\epsilon^{-2})$ sample complexity under partial data coverage.
For the partial data coverage assumption, we use the notion of feature coverage.
An interesting future work would be to design an algorithm that allows using a weaker notion of feature coverage in the sample complexity bound.

\section*{Acknowledgement}

We acknowledge the support of NSF via grant IIS-2007055.

\section*{Impact Statement}
This paper presents work whose goal is to advance the field of Reinforcement Learning Theory. There are potential societal consequences of our work, none which we feel must be specifically highlighted here.

\section*{References}
\printbibliography[heading=none]

\newpage
\appendix
\onecolumn

\section{Covering} \label{appendix:covering}

\begin{lemma}[Covering balls. e.g. \textcite{wainwright2019high}] \label{lemma:cover-ball}
For any $\epsilon \in (0, 1)$, we have
$$
\log \mathcal{N}(\mathbb{B}_d(r), \Vert \cdot \Vert_\infty, \epsilon) \leq d \log \left( 1 + \frac{2r}{\epsilon} \right).
$$
\end{lemma}

\begin{lemma}[Lemma 7 in \textcite{zanette2021provable}]
Consider a feature mapping $\bm\varphi : \mathcal{S} \times \mathcal{A} \rightarrow \mathbb{R}^d$ such that $\Vert \bm\varphi(s, a) \Vert_2 \leq 1$ for all $(s, a) \in \mathcal{S} \times \mathcal{A}$.
Then for all $s \in \mathcal{S}$, we have
$$
\sum_{a \in \mathcal{A}} \vert \pi_{\bm\theta'} (a | s) - \pi_{\bm\theta}(a | s) \vert \leq 8 \Vert \bm\theta - \bm\theta' \Vert_2
$$
for any pair $\bm\theta, \bm\theta' \in \mathbb{R}^d$ such that $\Vert \bm\theta - \bm\theta' \Vert_2 \leq \frac{1}{2}$.
\end{lemma}

\begin{lemma}[Covering softmax function class. Lemma 6 in \textcite{zanette2021provable}] \label{lemma:cover-pi}
  For any $\epsilon \in (0, 1)$, we have
  $$
  \log \mathcal{N}(\Pi(B), \Vert \cdot \Vert_{\infty, 1}, \epsilon) \leq d \log \left(1 + \frac{16 B}{\epsilon} \right)
  $$
  where the norm $\Vert \cdot \Vert_{\infty, 1}$ is defined by
  $$
  \Vert \pi - \pi' \Vert_{\infty, 1} \coloneqq \sup_{s \in \mathcal{S}} \sum_{a \in \mathcal{A}} \vert \pi(a | s) - \pi'(a | s) \vert.
  $$
\end{lemma}

\begin{lemma}[Covering number bound for the space of $v$] \label{lemma:covering}
Consider the function class
$$
\mathcal{V} = \left\{ v_{\bm\zeta, \pi} : \bm\zeta \in \mathbb{B}(D_\zeta), \pi \in \Pi(D_\pi) \right\}
$$
where $v_{\bm\zeta, \pi} : \mathcal{S} \rightarrow \mathbb{R}$ is defined by $v_{\bm\zeta, \pi}(s) = \sum_a \pi(a | s) \langle \bm\zeta, \varphi(s, a) \rangle$.
Then,
$$
\mathcal{N}(\mathcal{V}, \Vert \cdot \Vert_\infty, \epsilon) \leq
\mathcal{N}(\mathbb{B}(D_\zeta), \Vert \cdot \Vert_2, \epsilon / 2) \times
\mathcal{N}(\Pi(D_\pi), \Vert \cdot \Vert_{\infty, 1}, \epsilon / (2 D_\zeta)).
$$ and it follows that
$$
\log \mathcal{N}(\mathcal{V}, \Vert \cdot \Vert_\infty, \epsilon) \leq
\mathcal{O}(d \log (D_\zeta D_\pi / \epsilon)).
$$
\end{lemma}
\begin{proof}
Consider $\mathcal{C}_v = \{ v_{\bm\zeta, \pi} \in (\mathcal{S} \rightarrow [0, \frac{1}{1 - \gamma}]) : \bm\zeta \in \mathcal{C}_\zeta, \pi \in \mathcal{C}_\pi \}$ where $\mathcal{C}_\zeta$ is an $\epsilon / 2$-cover of $\mathbb{B}(D_\zeta)$ with respect to $\Vert \cdot \Vert_2$ and $\mathcal{C}_\pi$ is an $\epsilon / (2 D_\zeta)$-cover of $\Pi(D_\pi)$ with respect to $\Vert \cdot \Vert_{\infty, 1}$.
Such covers with $\vert \mathcal{C}_\zeta \vert \leq (1 + 4 D_\zeta / \epsilon)^d$ and $\vert \mathcal{C}_\pi \vert \leq (1 + 32 D_\zeta D_\pi / \epsilon)^d$ exist by previous lemmas.
Consider any $v_{\bm\zeta, \pi} \in \mathcal{V}$. Then, there exists $\bm\zeta' \in \mathcal{C}_\zeta$ and $\pi' \in \mathcal{C}_\pi$ with $\Vert \bm\zeta - \bm\zeta' \Vert_2 \leq \epsilon / 2$ and $\Vert \pi - \pi' \Vert_{\infty, 1} = \sup_{s \in \mathcal{S}} \sum_{a \in \mathcal{A}} \vert \pi(a | s) - \pi'(a | s) \vert \leq \epsilon / (2 D_\zeta)$.
Then for any $s \in \mathcal{S}$, $v_{\bm\zeta', \pi'} \in \mathcal{C}_v$ satisfies
\begin{align*}
\vert v_{\bm\zeta, \pi}(s) - v_{\bm\zeta', \pi'}(s) \vert
&=
\vert \sum_a \pi(a | s) \langle \bm\zeta, \varphi(s, a) \rangle - \sum_a \pi'(a | s) \langle \bm\zeta', \varphi(s, a) \rangle \vert \\
&=
\vert \sum_a (\pi(a | s) - \pi'(a | s)) \langle \bm\zeta, \varphi(s, a) \rangle + \pi'(a | s) \langle \bm\zeta - \bm\zeta', \varphi(s, a) \rangle \vert \\
&\leq
D_\zeta \sum_a \vert \pi(a | s) - \pi'(a | s) \vert + \sum_a \pi'(a | s) \epsilon / 2 \\
&\leq \epsilon.
\end{align*}
It follows that $\mathcal{C}_v$ is an $\epsilon$-cover of $\mathcal{V}$ with respect to $\Vert \cdot \Vert_\infty$ with $\vert \mathcal{C}_v \vert = \vert \mathcal{C}_\zeta \vert \vert \mathcal{C}_\pi \vert$ and we are done.
\end{proof}

\section{Concentration Inequalities}

\begin{lemma}[Matrix Bernstein] \label{lemma:berstein}
Consider a finite sequence $\{S_k\}$ of independent, random matrices with common dimension $d_1 \times d_2$.
Assume that $\mathbb{E} S_k = 0$ and $\Vert S_k \Vert \leq L$ for each index $k$.
Let $Z = \sum_k S_k$ and define
$v(Z) \coloneqq \max\{ \Vert \mathbb{E}[Z Z^T] \Vert, \Vert \mathbb{E}[Z^T Z] \Vert \} = \max\{ \Vert \sum_k \mathbb{E}[S_k S_k^T] \Vert, \Vert \sum_k \mathbb{E}[S_k^T S_k ] \Vert \}$.
Then,
$$
P(\Vert Z \Vert \geq t) \leq (d_1 + d_2) \exp \left(
  \frac{-t^2/2}{v(Z) + Lt/3}
\right)
$$
and it follows that with probability at least $1 - \delta$, we have
$$
\Vert Z \Vert \leq \frac{2 L \log((d_1 + d_2) / \delta)}{3} + \sqrt{2 v(Z) \log ((d_1 + d_2) / \delta)}.
$$
\end{lemma}

\begin{lemma} \label{lemma:impossible}
There exists an example where $\bm\lambda^\ast = \mathbb{E}_{\mu^\ast} [ \bm\varphi(s, a) ]$ is not in the span of $\bm\varphi(s_1, a_1), \dots, \bm\varphi(s_n, a_n)$ with probability at least $1 / 2$.
\end{lemma}
\begin{proof}
Consider the case where $\mathcal{S} = \{ s \}$, $\mathcal{A} = \{ a_1, a_2 \}$, $d = 2$, and $\bm\varphi(s, a_1) = \bm{e}_1$ and $\bm\varphi(s, a_2) = \bm{e}_2$.
Let $\mu = \mu^\ast$ and $\mu(s, a_1) = p$ and $\mu(s, a_2) = 1 - p$.
Let $F$ be the event where $\bm\lambda^\ast$ is not in the span of $\bm\varphi(s_1, a_1), \dots, \bm\varphi(s_n, a_n)$.
Then,
$$
P(F) = (1 - p)^n \geq \frac{1}{2}
$$
as long as we choose $p \leq 1 - 2^{-1 / n}$.
\end{proof}

\begin{proof}[Proof of Lemma~\ref{lemma:lambda-hat}]
Let $c_k = w^\ast(s_k, a_k) = \mu^\ast(s_k, a_k) / \mu(s_k, a_k)$, $k = 1, \dots, n$.
Note that $\mu(s_k, a_k) > 0$, $k = 1, \dots, n$ must hold, otherwise such $s_k, a_k$ cannot be sampled.
By the concentrability assumption, we have $c_k \in [0, C^\ast]$, $k = 1, \dots, n$.
Let $\bm{z}_k = c_k \bm\varphi(s_k, a_k)$ for $k = 1, \dots, n$.
Then, $\Vert \bm{z}_k \Vert \leq C^\ast$ and
$$
\mathbb{E} [ \bm{z}_k ] = \mathbb{E}_{(s, a) \sim \mu} [ w^\ast(s, a) \bm\varphi(s, a) ] = \mathbb{E}_{(s, a) \sim \mu} \left[\frac{\mu^\ast(s, a)}{\mu(s, a)} \bm\varphi(s, a)\right] = \mathbb{E}_{(s, a) \sim \mu^\ast} [ \bm\varphi(s, a) ] = \bm\lambda^\ast.
$$
Define $\bm{S}_k = \bm{z}_k - \bm\lambda^\ast$, $k = 1, \dots, n$
Then, $\mathbb{E} [ \bm{S}_k ] = 0$ and $\Vert \bm{S}_k \Vert_2 \leq \Vert \bm{z}_k \Vert_2 + \mathbb{E}[\Vert \bm{z}_k \Vert_2] \leq 2 C^\ast$ and $\Vert \mathbb{E}[ \bm{S}_k^T \bm{S}_k ] \Vert_2 \leq \mathbb{E}[\bm{z}_k^T \bm{z}_k] \leq (C^\ast)^2$ and $\Vert \mathbb{E}[\bm{S}_k \bm{S}_k^T ] \Vert_2 \leq (C^\ast)^2$.
Applying matrix Bernstein inequality (Lemma~\ref{lemma:berstein}) on $\{ \bm{S}_k \}_{k = 1}^n$, we have
\begin{align*}
\Vert \frac{1}{n} \sum_{k = 1}^n \bm{S}_k \Vert_2 = \Vert \frac{1}{n} w^\ast(s_k, a_k) \bm\varphi(s_k, a_k) - \bm\lambda^\ast \Vert_2
\leq \frac{4 C^\ast \log((d + 1) / \delta)}{3n} + \sqrt{\frac{8 (C^\ast)^2 \log((d + 1) / \delta)}{n}}
\end{align*}
with probability at least $1 - \delta$ and the result follows.
\end{proof}

\begin{lemma} \label{lemma:lambda-hat2}
Under the feature coverage assumption~\ref{assumption:feature-coverage}, there exists $\widehat{\bm\lambda}^\ast \in \mathbb{R}^d$ of the form $\widehat{\bm\lambda}^\ast = \frac{1}{n} \sum_{k = 1}^n c_k \bm\varphi(s_k, a_k)$ with $c_k \in [0, C^\ast]$, $k = 1, \dots, n$ such that
$$
\Vert \widehat{\bm\lambda}^\ast - \bm\lambda^\ast \Vert_2 \leq \mathcal{O}\left(C^\ast \sqrt{\frac{\log(d / \delta)}{n}} \right)
$$
with probability at least $1 - \delta$.
\end{lemma}
\begin{proof}
Since $\bm\lambda^\ast \in \text{Col}(\bm\Lambda)$, we have $\bm\lambda^\ast = \bm\Lambda \bm\Lambda^\dagger \bm\lambda^\ast$ and it follows that
\begin{align*}
\bm\lambda^\ast
&= \bm\Lambda \bm\Lambda^\dagger \bm\lambda^\ast \\
&= \mathbb{E} [ \bm\varphi(s, a) \bm\varphi(s, a)^T \bm\Lambda^\dagger \bm\lambda^\ast ] \\
&= \mathbb{E} [ \langle \bm\varphi(s, a), \bm\Lambda^\dagger \bm\lambda^\ast \rangle \bm\varphi(s, a) ].
\end{align*}
Let $c_k = \langle \bm\varphi(s_k, a_k), \bm\Lambda^\dagger \bm\lambda^\ast \rangle$.
Then, by Cauchy-Schwartz, we have
$$
\vert c_k \vert \leq \Vert \bm\varphi(s_k, a_k) \Vert_2 \Vert \bm\Lambda^\dagger \bm\lambda^\ast \Vert_2 \leq C^\ast
$$
where the last inequality follows by the feature coverage assumption.
Using matrix Bernstein inequality as is done in the proof of Lemma~\ref{lemma:lambda-hat}, the result follows.
\end{proof}

\subsection{Proof of Lemma~\ref{lemma:confidence1}} \label{appendix:confidence1}

\begin{proof}[Proof of Lemma~\ref{lemma:confidence1}]
Consider $\bm\lambda = \frac{1}{n} \sum_{k = 1}^n c_k \bm\varphi(s_k, a_k)$ with $\vert c_k \vert \leq B$, $k = 1, \dots, n$.
Let $\widehat{\bm\Lambda}_n = \bm{U} \bm{D} \bm{U}^T$ be the eigendecomposition of $\widehat{\bm\Lambda}_n = \frac{1}{n} \sum_{k = 1}^n \bm\varphi(s_k, a_k) \bm\varphi(s_k, a_k)^T$ where $\bm{D} = \text{diag}(d_1, \dots, d_d)$ with $d_1 \geq \dots \geq d_d \geq 0$.
Then, we have $\bm{D} = \bm{U}^T \widehat{\bm\Lambda}_n \bm{U} = \frac{1}{n} \sum_{k = 1}^n \bm{U}^T \bm\varphi(s_k, a_k) \bm\varphi(s_k, a_k)^T \bm{U}$ and it follows that $d_i = \frac{1}{n} \sum_{k = 1}^n \langle \bm{u}_i, \bm\varphi(s_k, a_k) \rangle^2$ where $d_i$ is the $i$th diagonal entry of $\bm{D}$.
Also,
\begin{align*}
\bm\lambda
&= \frac{1}{n} \sum_{k = 1}^n c_k \bm\varphi(s_k, a_k) \\
&= \frac{1}{n} \sum_{k = 1}^n c_k \sum_{i = 1}^d \langle \bm\varphi(s_k, a_k), \bm{u}_i \rangle \bm{u}_i \\
&= \sum_{i = 1}^d \left( \frac{1}{n} \sum_{k = 1}^n c_k \langle \bm\varphi(s_k, a_k), \bm{u}_i \rangle \right) \bm{u}_i.
\end{align*}
where the second equality follows by $\bm{x} = \bm{U} \bm{U}^T \bm{x} = \sum_{i = 1}^d \langle \bm{x}, \bm{u}_i \rangle \bm{u}_i$ for any vector $\bm{x} \in \mathbb{R}^d$.
So,
\begin{align*}
\bm\lambda^T \widehat{\bm\Lambda}_n^\dagger \bm\lambda
&= \sum_{i = 1}^{d'} \left(\frac{1}{n} \sum_{k = 1}^n c_k \langle \bm\varphi(s_k, a_k), \bm{u}_i \rangle \right)^2 / d_i \\
&= \frac{1}{n} \sum_{i = 1}^{d'} \left( \sum_{k = 1}^n c_k \langle \bm\varphi(s_k, a_k), \bm{u}_i \rangle \right)^2 \Bigg/ \left(\sum_{k = 1}^n \langle \bm\varphi(s_k, a_k), \bm{u}_i \rangle^2 \right) \\
&\leq
\frac{1}{n} \sum_{i = 1}^{d'} \left( \sum_{k = 1}^n c_k^2 \right) \\
&\leq d B^2
\end{align*}
where $d'$ is the number of strictly positive diagonal entries in $\bm{D}$ and the first inequality follows by Cauchy-Schwartz.
\end{proof}

\subsection{Proof of Lemma~\ref{lemma:ls-uniform}}

\begin{lemma} \label{lemma:ls}
Let $v : \mathcal{S} \rightarrow [0, D_v]$.
With probability at least $1 - \delta$, we have
$$
\Vert \bm\Psi \bm{v} - \widehat{\bm\Psi \bm{v}} \Vert_{n \widehat{\bm\Lambda}_n + \bm{I}}
\leq \mathcal{O}\left(
D_v \sqrt{d \log (n / \delta)}
\right)
$$
where $\widehat{\bm\Psi \bm{v}}$ is the least squares estimate defined in \eqref{eqn:least-squares}.
\end{lemma}
\begin{proof}
Note that $\Vert \bm\Psi \bm{v} \Vert_2 \leq D_v \sqrt{d}$ by the boundedness assumption on $\bm\Psi$.
The result follows directly from Theorem 2 in \textcite{abbasi2011improved}.
\end{proof}

\begin{proof}[Proof of Lemma~\ref{lemma:ls-uniform}]
Let $\mathcal{C}$ be an $(1 / n)$-cover on $\mathcal{V}$.
By Lemma~\ref{lemma:covering}, such a cover with $\log \vert \mathcal{C} \vert \leq \mathcal{O}(d \log(D_\zeta D_\pi n))$ exists.
Applying a union bound over $\mathcal{C}$ and using the concentration bound in Lemma~\ref{lemma:ls}, we get
\begin{equation} \label{eqn:self-normalizing}
\left\Vert \bm\Psi \bm{v} - \widehat{\bm\Psi \bm{v}} \right\Vert_{n \widehat{\bm\Lambda}_n + \bm{I}} \leq \mathcal{O} \left( D_v \sqrt{d \log (D_\zeta D_\pi n / \delta)} \right)
\end{equation}
for all $\bm{v} \in \mathcal{C}$ with probability at least $1 - \delta$.
For any $\bm{v} \in \mathcal{V}$, we can find $v'$ in the cover that satisfies $\Vert \bm{v} - \bm{v}' \Vert_\infty \leq 1 / n$. Hence,
\begin{align*}
\left\Vert \bm\Psi \bm{v} - \widehat{\bm\Psi \bm{v}} \right\Vert_{n \widehat{\bm\Lambda}_n + \bm{I}}
&\leq
\Vert \bm\Psi (\bm{v} - \bm{v}') + \bm\Psi \bm{v}' - \widehat{\bm\Psi \bm{v}'} + \widehat{\bm\Psi \bm{v}'} - \widehat{\bm\Psi \bm{v}} \Vert_{n \widehat{\bm\Lambda}_n + \bm{I}} \\
&\leq
\Vert \bm\Psi(\bm{v} - \bm{v}') \Vert_{n \widehat{\bm\Lambda}_n + \bm{I}}
+ \left\Vert (n \widehat{\bm\Lambda}_n + \bm{I})^{-1} \sum_{k = 1}^n (v'(s'_k) - v(s_k')) \bm\varphi(s_k, a_k) \right\Vert_{n \widehat{\bm\Lambda}_n + \bm{I}} \\
&\hspace{10mm}+ \mathcal{O}\left( D_v \sqrt{d \log(D_\zeta D_\pi n / \delta) }\right).
\end{align*}

The first term can be bounded using the boundedness assumption on $\bm\Psi$ by
$$
\Vert \bm\Psi(\bm{v} - \bm{v}') \Vert_{n \widehat{\bm\Lambda}_n + \bm{I}}^2 \leq \Vert \bm\Psi(\bm{v} - \bm{v}') \Vert_2^2 \Vert n \widehat{\bm\Lambda}_n + \bm{I} \Vert_2
\leq d / n^2 \cdot (1 + n) \leq \mathcal{O}(1)
$$
as long as $n \geq d$.
The second term can be bounded by
\begin{align*}
&\left\Vert (n \widehat{\bm\Lambda} + \bm{I})^{-1} \sum_{k = 1}^n (v'(s'_k) - v(s'_k)) \bm\varphi(s_k, a_k) \right\Vert_{n \widehat{\bm\Lambda} + \bm{I}}^2 \\
&\hspace{10mm}=
\sum_{k = 1}^n (v'(s_k') - v(s_k')) \bm\varphi(s_k, a_k)^T (n \widehat{\bm\Lambda} + \bm{I})^{-1} \sum_{k = 1}^n (v'(s'_k) - v(s'_k)) \bm\varphi(s_k, a_k) \\
&\hspace{10mm}\leq
\sum_{k = 1}^n \Vert (v'(s'_k) - v(s'_k)) \bm\varphi(s_k, a_k) \Vert_2^2 \\
&\hspace{10mm}\leq
\sum_{k = 1}^n (v'(s'_k) - v(s'_k))^2 \\
&\hspace{10mm}\leq 1
\end{align*}
where the first inequality uses $n \widehat{\bm\Lambda} + \bm{I} \succcurlyeq \bm{I}$.
The result follows.
\end{proof}

\section{Computational Efficiency} \label{appendix:computation}

In this section, we explain why our algorithms are computationally efficient by showing that the algorithms only require computing quantities for states that appear in the offline dataset to compute the policy $\pi_t$ at each step.
This is how we avoid computation complexity that scales with the size of the state space.

Recall that $\pi_t = \sigma(\alpha \sum_{i = 1}^{t - 1} \bm\Phi \bm\zeta_i)$ and by definition of $\sigma(\cdot)$,
$$
\pi_t(a | s) = \frac{\exp(\alpha \sum_{i = 1}^{t - 1} \bm\varphi(s, a)^T \bm\zeta_i )}{\sum_{a'} \exp( \alpha \sum_{i = 1}^{t - 1} \bm\varphi(s, a')^T \bm\zeta_i )}.
$$

We argue that the algorithm only needs to compute $\pi_t(a | s)$ for the states $s$ that appear as the next state in the dataset $\mathcal{D}$. There are two parts where the object $\pi_t$ is used in the algorithm:

\paragraph{Line~\ref{alg-line:lambda-unconstrained} in Algorithm~\ref{alg:pdapc} and Line~\ref{alg-line:lambda-constrained} in Algorithm~\ref{alg:pdapc-constrained}}
In these lines, the object $\pi_t$ is used to compute
$$
\widehat{\bm\Psi \bm{v}}_{\bm\zeta_t, \pi_t} = (n \widehat\Lambda + I)^{-1} \sum_{k = 1}^n v_{\bm\zeta_t, \pi_t} (s_k') \bm\varphi(s_k, a_k)
$$
where $\bm{v}_{\bm\zeta_t, \pi_t}(s_k') = (n \widehat\Lambda + I)^{-1} \sum_a \pi_t(a | s_k') \langle \bm\zeta_t, \bm\varphi(s_k', a) \rangle$.
As we claimed, we only need to compute $\pi_t(\cdot | s_k')$ for $s_k'$ that appear in the dataset $\mathcal{D}$.

\paragraph{Line~\ref{alg-line:zeta-unconstrained} in Algorithm \ref{alg:pdapc} and Line~\ref{alg-line:zeta-constrained} in Algorithm \ref{alg:pdapc-constrained}}
In this lines, the object $\pi_t$ is used to compute $\bm\Phi^T \widehat{\bm\mu}_{\bm\lambda_t, \pi_t}$ for $\bm\lambda_t$ of the form $\bm\lambda_t = \frac{1}{n} \sum_{k = 1}^n c_k \bm\varphi(s_k, a_k)$. By definition,
$$
\widehat\mu_{\lambda, \pi} = \pi \circ E[(1 - \gamma) \nu_0 + \gamma \widehat{\Psi \lambda} = \pi \circ E [ (1 - \gamma) e_{s_0} + \gamma \frac{1}{n} \sum_{k = 1}^n c_k e_{s_k'} ]
$$
and it follows that
\begin{align*}
\bm\Phi^T \widehat{\bm\mu}_{\bm\lambda_t, \pi_t}
&= (1 - \gamma) \bm\Phi^T (\pi_t \circ \bm{E} \bm{e}_{s_0}) + \gamma \frac{1}{n} \sum_{k = 1}^n c_k \bm\Phi^T (\pi_t \circ \bm{E} \bm{e}_{s_k'}) \\
&= (1 - \gamma) \sum_a \pi_t(s_0, a) \bm\varphi(s_0, a) + \gamma \frac{1}{n} \sum_{k = 1}^n c_k \sum_a \pi_t(a | s_k') \bm\varphi(s_k', a).
\end{align*}
Again, we only need to compute $\pi_t(\cdot | s_k')$ for $s_k'$ that appears in $\mathcal{D}$.

\section{Details in Offline Unconstrained RL Setting}

\subsection{Bounding the Regret of $\pi$-Player} \label{appendix:pi-player}

\begin{lemma}[Mirror Descent, Lemma D.2 in \textcite{gabbianelli2023offline}]
Let $q_1, \dots, q_T$ be a sequence of functions from $\mathcal{S} \times \mathcal{A}$ to $\mathbb{R}$ with $\Vert q_t \Vert_\infty \leq D_q$ for $t = 1, \dots, T$.
Given an initial policy $\pi_1 : \mathcal{S} \rightarrow \Delta(\mathcal{A})$ and a learning rate $\alpha > 0$, define the sequence of policies $\pi_2, \dots, \pi_{T + 1}$ such that
$$
\pi_{t + 1}(a | s) \propto \pi_t(a | s) \exp(\alpha q_t(s, a)).
$$
Then, for any comparator policy $\pi^\ast$, we have
$$
\sum_{t = 1}^T \sum_{s \in \mathcal{S}} \nu^{\pi^\ast}(s) \langle \pi^\ast( \cdot | s )  - \pi_t( \cdot | s), q_t(s, \cdot) \rangle
\leq
\frac{\mathcal{H}(\pi^\ast \Vert \pi_1 )}{\alpha} + \frac{\alpha T D_q^2}{2}
$$
where $\mathcal{H}(\pi \Vert \pi') \coloneqq \sum_{s \in \mathcal{S}} \nu^\pi(s) \mathcal{D}(\pi(\cdot | s) \Vert \pi'(\cdot | s))$ is the conditional entropy.
\end{lemma}

\begin{lemma}[Lemma B.3 in \textcite{gabbianelli2023offline}] \label{lemma:mirror}
The sequence of policies $\pi_1, \dots, \pi_T$ produced by an exponentiation algorithm $\pi_{t + 1} = \sigma(\alpha \sum_{i = 1}^t \bm\Phi \bm\zeta_i)$ satisfies
$$
\sum_{t = 1}^T \sum_{s \in \mathcal{S}} \nu^{\pi^\ast}(s) \sum_{a \in \mathcal{A}} (\pi^\ast(a | s) - \pi_t(a | s)) \langle \bm\zeta_t, \bm\varphi(s, a) \rangle \leq \frac{\log \vert \mathcal{A} \vert}{\alpha} + \frac{\alpha T D_\varphi^2 D_\zeta^2}{2}
$$
where $\Vert \bm\zeta_t \Vert_2 \leq D_\zeta$, $t = 1, \dots, T$ and $\Vert \bm\varphi(\cdot, \cdot) \Vert_2 \leq D_\varphi$.
\end{lemma}

\subsection{Bounding the regret of $\zeta$-player} \label{appendix:zeta-regret}

\begin{proof}[Proof of Lemma~\ref{lemma:phiTmu}]
Recall $\mu_{\bm\lambda(\bm{c}), \pi}(s, a) = \pi(a | s) \left[(1 - \gamma) \nu_0(s) + \gamma (\bm\Psi^T \bm\lambda(\bm{c}))_s \right]$ where we use the notation $(\bm{x})_s$ to denote the $s$th entry of vector $\bm{x}$.
We can write
\begin{align*}
\bm\Phi^T \bm\mu_{\bm\lambda(\bm{c}), \pi}
&= \sum_{s, a} \mu_{\bm\lambda(\bm{c}), \pi}(s, a) \bm\varphi(s, a) \\
&= \sum_{s, a} \pi(a | s) [ (1 - \gamma) \nu_0(s) + \gamma (\bm\Psi^T \bm\lambda(\bm{c}))_s ] \bm\varphi(s, a) \\
&= (1 - \gamma) \sum_a \pi(a | s_0) \bm\varphi(s_0, a) + \gamma \sum_{s} (\bm\Psi^T \bm\lambda(\bm{c}))_s \sum_a \pi(a | s) \bm\varphi(s, a) \\
&= (1 - \gamma) \bm\varphi(s_0, \pi) + \gamma \sum_s (\bm\Psi^T \bm\lambda(\bm{c}))_s \bm\varphi(s, \pi)
\end{align*}
where we use the notation $\bm\varphi(s, \pi) = \sum_a \pi(a | s) \bm\varphi(s, a)$.
Recall that $\bm\lambda(\bm{c}) = \frac{1}{n} \sum_{k = 1}^n c_k \bm\varphi(s_k, a_k)$ where $c_k \in [-B, B]$, $k = 1, \dots, n$.
Following the same argument for expanding $\bm\Phi^T \bm\mu_{\bm\lambda(\bm{c}), \pi}$, we get
\begin{align*}
\bm\Phi^T \widehat{\bm\mu}_{\bm\lambda(\bm{c}), \pi}
&=
(1 - \gamma) \bm\varphi(s_0, \pi) + \gamma \sum_s (\widehat{\bm\Psi^T \bm\lambda}(\bm{c}))_s \bm\varphi(s, \pi) \\
&=
(1 - \gamma) \bm\varphi(s_0, \pi)+ \frac{\gamma}{n} \sum_{k = 1}^n c_k \bm\varphi(s_k', \pi).
\end{align*}
Also, using $\bm\Psi^T \bm\lambda(\bm{c}) = \frac{1}{n} \sum_{k = 1}^n c_k \bm\Psi^T \bm\varphi(s_k, a_k) = \frac{1}{n} \sum_{k = 1}^n c_k P(\cdot | s_k, a_k) = \frac{1}{n} \sum_{k = 1}^n c_k \mathbb{E}[\bm{e}_{s_k'} | s_k, a_k]$, we get
\begin{align*}
\bm\Phi^T \bm\mu_{\bm\lambda(\bm{c}), \pi}
&=
(1 - \gamma) \bm\varphi(s_0, \pi) + \gamma \sum_s (\bm\Psi^T \bm\lambda(\bm{c}))_s \bm\varphi(s, \pi) \\
&=
(1 - \gamma) \bm\varphi(s_0, \pi) + \frac{\gamma}{n} \sum_{k = 1}^n c_k \mathbb{E} [ \bm\varphi(s_k', \pi) | s_k, a_k ].
\end{align*}
Hence,
\begin{align*}
\Vert \bm\Phi^T(\bm\mu_{\bm\lambda(\bm{c}), \pi} - \widehat{\bm\mu}_{\bm\lambda(\bm{c}), \pi}) \Vert_2
&= \gamma \left\Vert \frac{1}{n} \sum_{k = 1}^n c_k (\bm\varphi(s_k', \pi) - \mathbb{E}[\bm\varphi(s_k', \pi) | s_k, a_k ]) \right\Vert_2 \\
&\leq \mathcal{O}\left(B \sqrt{\frac{\log(d / \delta)}{n}}\right)
\end{align*}
where the last inequality uses Matrix Bernstein inequality (Lemma~\ref{lemma:berstein}) with $S_k = c_k \bm\varphi(s_k', \pi) - c_k \mathbb{E}[\bm\varphi(s_k', \pi) | s_k, a_k]$.
\end{proof}

\begin{lemma} \label{lemma:phiTmu-pi}
Given a fixed $\bm\lambda \in \mathcal{C}_n(B)$, we have for all $\pi \in \Pi(D_\pi)$ that
$$
\Vert \bm\Phi^T \bm\mu_{\bm\lambda, \pi} - \bm\Phi^T \widehat{\bm\mu}_{\bm\lambda, \pi} \Vert_2 \leq \mathcal{O}\left(B \sqrt{\frac{\log(d / \delta) + d \log (D_\pi d n)}{n}}\right)
$$
with probability at least $1 - \delta$.
\end{lemma}
\begin{proof}
Consider an $\varepsilon$-cover of $\Pi(D_\pi)$ with covering balls    when measuring distances with the norm $\Vert \pi - \pi' \Vert_{\infty, 1} = \sup_{s \in \mathcal{S}} \sum_{a \in \mathcal{A}} \vert \pi(a | s) - \pi'(a | s) \vert$.
By Lemma~\ref{lemma:cover-pi}, there exists such a cover with log covering number
$$
\log \mathcal{N}(\Pi(D_\pi), \Vert \cdot \Vert_{\infty, 1}, \varepsilon) \leq d \log \left(1 + \frac{16 D_\pi}{\varepsilon} \right).
$$
Fix any $\pi \in \Pi(D_\pi)$ and consider its nearest cover center $\pi'$ measuring distances by $\Vert \cdot \Vert_{\infty, 1}$.
Then,
\begin{align*}
\Vert \bm\Phi^T \bm\mu_{\bm\lambda, \pi} - \bm\Phi^T \widehat{\bm\mu}_{\bm\lambda, \pi} \Vert_2
&\leq
\Vert \underbrace{\bm\Phi^T \bm\mu_{\bm\lambda, \pi} - \bm\Phi^T \bm\mu_{\bm\lambda, \pi'}}_{(i)} \Vert_2 + \Vert \underbrace{\bm\Phi^T \bm\mu_{\bm\lambda, \pi'} - \bm\Phi^T \widehat{\bm\mu}_{\bm\lambda, \pi'}}_{(ii)} \Vert_2 + \Vert \underbrace{\bm\Phi^T \widehat{\bm\mu}_{\bm\lambda, \pi'} - \bm\Phi^T \widehat{\bm\mu}_{\bm\lambda, \pi}}_{(iii)} \Vert_2.
\end{align*}
Note that
\begin{align}
\bm\Phi^T \bm\mu_{\bm\lambda, \pi}
&= \sum_{s, a} \mu_{\bm\lambda, \pi}(s, a) \bm\varphi(s, a) \notag \\
&= \sum_{s, a} \pi(a | s) [ (1 - \gamma) \nu_0(s) + \gamma (\bm\Psi^T \bm\lambda)_s ] \bm\varphi(s, a) \notag \\
&= (1 - \gamma) \sum_a \pi(a | s_0) \bm\varphi(s_0, a) + \gamma \sum_{s} (\bm\Psi^T \bm\lambda)_s \sum_a \pi(a | s) \bm\varphi(s, a) \notag \\
&= (1 - \gamma) \bm\varphi(s_0, \pi) + \gamma \sum_s (\bm\Psi^T \bm\lambda)_s \bm\varphi(s, \pi), \label{eqn:phi-mu}
\end{align}
where we use the notation $\bm\varphi(s, \pi) = \sum_a \pi(a | s) \bm\varphi(s, a)$.
The first term $(i)$ can be bounded by
\begin{align*}
\Vert \bm\Phi^T \bm\mu_{\bm\lambda, \pi} &- \bm\Phi^T \bm\mu_{\bm\lambda, \pi'} \Vert_2
\leq (1 - \gamma) \left\Vert \bm\varphi(s_0, \pi - \pi') \right\Vert_2 + \gamma \left\Vert \sum_s (\bm\Psi^T \bm\lambda)_s \bm\varphi(s, \pi - \pi') \right\Vert_2 \\
&= (1 - \gamma) \left\Vert \sum_a (\pi(a | s_0) - \pi'(a | s_0)) \bm\varphi(s_0, a) \right\Vert_2 + \gamma \left\Vert \sum_s (\bm\Phi^T \bm\lambda)_s \sum_a (\pi(a | s) - \pi'(a | s)) \bm\varphi(s, a) \right\Vert_2 \\
&\leq (1 - \gamma) \sum_a \vert \pi(a | s_0) - \pi'(a | s_0) \vert \Vert \bm\varphi(s_0, a) \Vert_2
+ \gamma \sum_s (\vert \bm\Psi \vert^T \vert \bm\lambda \vert)_s \sum_a \vert \pi(a | s) - \pi'(a | s) \vert \Vert \bm\varphi(s, a) \Vert_2 \\
&\leq (1 - \gamma) \varepsilon + \gamma \varepsilon \bm{1}_\mathcal{S}^T \vert \bm\Psi \vert^T \vert \bm\lambda \vert \\
&\leq
(1 - \gamma) \varepsilon + \gamma \varepsilon D_\psi \sqrt{d} \Vert \bm\lambda \Vert_2 \\
&\leq
(1 - \gamma) \varepsilon + \gamma \varepsilon \sqrt{d} D_\psi B
\end{align*}
where the second to last inequality uses the boundedness assumption on $\bm\Psi$ and the last inequality follows by $\Vert \bm\lambda \Vert_2 \leq B$.

The second term $(ii)$ can be bounded by a union bound of the concentration inequality in Lemma~\ref{lemma:phiTmu} across all $\pi'$ in the cover, resulting in the following bound
$$
\Vert \bm\Phi^T \bm\mu_{\bm\lambda, \pi'} - \bm\Phi^T \widehat{\bm\mu}_{\bm\lambda, \pi'} \Vert_2 \leq \mathcal{O}\left( B \sqrt{\frac{\log(d / \delta) + d \log(D_\pi / \varepsilon)}{n}} \right).
$$

To bound the third term $(iii)$, note that
\begin{align}
\bm\Phi^T \widehat{\bm\mu}_{\bm\lambda, \pi}
&=
(1 - \gamma) \bm\varphi(s_0, \pi) + \gamma \sum_s (\widehat{\bm\Psi^T \bm\lambda})_s \bm\varphi(s, \pi) \notag \\
&=
(1 - \gamma) \bm\varphi(s_0, \pi)+ \frac{\gamma}{n} \sum_{k = 1}^n c_k \bm\varphi(s_k', \pi). \label{eqn:phiTmuhat}
\end{align}
where $\bm\lambda = \frac{1}{n} \sum_{k = 1}^n c_k \bm\varphi(s_k, a_k)$.
Therefore,
\begin{align*}
\Vert \bm\Phi^T \widehat{\bm\mu}_{\bm\lambda, \pi'} - \bm\Phi^T \widehat{\bm\mu}_{\bm\lambda, \pi} \Vert_2
&\leq
(1 - \gamma) \Vert \bm\varphi(s_0, \pi - \pi') \Vert_2 + \frac{\gamma}{n} \sum_{k = 1}^n c_k \Vert \bm\varphi(s_k', \pi - \pi') \Vert_2 \\
&\leq
(1 - \gamma) \varepsilon + \gamma B \varepsilon
\end{align*}
where the last inequality uses $\Vert \bm\varphi(s, \pi - \pi') \Vert_2 = \Vert \sum_a (\pi(a | s) - \pi'(a | s)) \bm\varphi(s, a) \Vert_2 \leq \sum_a \vert \pi(a | s) - \pi'(a | s) \vert \Vert \bm\varphi(s, a) \Vert_2 \leq \varepsilon$.
Combining the bounds of the three terms, we get
$$
\Vert \bm\Phi^T \bm\mu_{\bm\lambda, \pi} - \bm\Phi^T \widehat{\bm\mu}_{\bm\lambda, \pi} \Vert_2 \leq \mathcal{O}\left(B \sqrt{\frac{\log(d / \delta) + d \log(D_\pi / \varepsilon)}{n}} + \sqrt{d} B \varepsilon \right).
$$
Choosing $\varepsilon = 1 / \sqrt{dn}$, we get the desired result.
\end{proof}

\begin{lemma} \label{lemma:regret-zeta}
The sequences $\{ \pi_t \}, \{ \bm\lambda(\bm{c}'_t) \}, \{\bm\zeta_t \}$ produced by Algorithm~\ref{alg:pdapc} satisfies
$$
\textsc{Reg}^\zeta_t = \langle \bm\zeta_t - \bm\zeta^{\pi_t}, \bm\Phi^T \bm\mu_{\bm\lambda(\bm{c}'_t), \pi_t} - \bm\lambda(\bm{c}'_t) \rangle \leq \mathcal{O}\left(
\frac{C^\ast d}{1 - \gamma} \sqrt{\frac{\log (d n T (\log \vert \mathcal{A} \vert) / \delta)}{n}}
\right)
$$
for $t = 1, \dots, T$ with probability at least $1 - \delta$.
\end{lemma}
\begin{proof}
Recall that $\mathcal{I} \subseteq \{1, \dots, n\}$ is an index set of size $d$ such that $\{ \bm\varphi(s_j, a_j) \}_{j \in \mathcal{I}}$ is a 2-approximate barycentric spanner for $\{ \bm\varphi(s_k, a_k) \}_{k = 1}^n$.
Let $\mathcal{C}_n'(C^\ast) = \{ \bm{c}' \in [-2C^\ast, 2C^\ast]^n : c_j' = 0 ~\text{if}~ j \in \mathcal{I} \}$.
Consider a $\varepsilon$-cover of $\mathcal{C}'_n(C^\ast)$ with respect to distance induced by $\Vert \cdot \Vert_\infty$ where $\varepsilon$ is to be chosen later, and let $\bm{c}''_t$ be the closest covering center to $\bm{c}'_t$.
There exists a cover with covering number $(1 + 4 C^\ast / \varepsilon)^d$.
We can decompose the regret of $\zeta$-player at step $t$ into
\begin{align*}
\textsc{Reg}_t^\zeta
&= \langle \bm\zeta_t - \bm\zeta^{\pi_t}, \bm\Phi^T \bm\mu_{\bm\lambda(\bm{c}_t'), \pi_t} - \bm\lambda(\bm{c}_t') \rangle \\
&= \underbrace{\langle \bm\zeta_t - \bm\zeta^{\pi_t}, \bm\Phi^T \bm\mu_{\bm\lambda(\bm{c}_t'), \pi_t} - \bm\Phi^T \bm\mu_{\bm\lambda_t(\bm{c}''_t), \pi_t} \rangle}_{(a)}
+ \underbrace{\langle \bm\zeta_t - \bm\zeta^{\pi_t}, \bm\Phi^T \bm\mu_{\bm\lambda(\bm{c}''_t), \pi_t} - \bm\Phi^T \widehat{\bm\mu}_{\bm\lambda(\bm{c}''_t), \pi_t} \rangle}_{(b)} \\
&\quad\quad+ \underbrace{\langle \bm\zeta_t - \bm\zeta^{\pi_t}, \bm\Phi^T \widehat{\bm\mu}_{\bm\lambda(\bm{c}_t''), \pi_t} - \bm\Phi^T \widehat{\bm\mu}_{\bm\lambda(\bm{c}_t'), \pi_t} \rangle}_{(c)}
+ \underbrace{\langle \bm\zeta_t - \bm\zeta^{\pi_t}, \bm\Phi^T \widehat{\bm\mu}_{\bm\lambda(\bm{c}_t'), \pi_t} - \bm\lambda(\bm{c}_t') \rangle}_{(d)}.
\end{align*}

\paragraph{Bounding $(a)$}
Recall from equation \eqref{eqn:phi-mu} that
$$
\bm\Phi^T \bm\mu_{\bm\lambda, \pi} = (1 - \gamma) \bm\varphi(s_0, \pi) + \gamma \sum_s (\bm\Psi^T \bm\lambda)_s \bm\varphi(s, \pi)
$$
where we use the notation $\bm\varphi(s, \pi) = \sum_a \pi(a | s) \bm\varphi(s, a)$.
Also, since $\Vert \bm{c}_t' - \bm{c}_t'' \Vert_\infty \leq \varepsilon$, we have
$$
\left\Vert \bm\lambda(\bm{c}_t') - \bm\lambda(\bm{c}_t'') \right\Vert_2 = \frac{1}{n} \left\Vert \sum_{k = 1}^n (c_{tk}' - c_{tk}'') \bm\varphi(s_k, a_k) \right\Vert_2
\leq
\frac{1}{n} \sum_{k = 1}^n \vert c_{tk}' - c_{tk}'' \vert \Vert \bm\varphi(s_k, a_k) \Vert_2 \leq \varepsilon.
$$

Hence,
\begin{align*}
\Vert \bm\Phi^T \bm\mu_{\bm\lambda(\bm{c}_t'), \pi_t} - \bm\Phi^T \bm\mu_{\bm\lambda(\bm{c}_t''), \pi_t} \Vert_2
&=
\gamma \Vert \sum_{s \in \mathcal{S}} (\bm\Psi^T (\bm\lambda(\bm{c}_t') - \bm\lambda(\bm{c}_t'')))_s \bm\varphi(s, \pi) \Vert \\
&\leq
\gamma \sum_{s \in \mathcal{S}} \vert (\bm\Psi^T(\bm\lambda(\bm{c}_t') - \bm\lambda(\bm{c}_t'')))_s \vert \Vert \bm\varphi(s, \pi) \Vert_2 \\
&\leq
\gamma \sum_{s \in \mathcal{S}} \vert (\bm\Psi^T (\bm\lambda(\bm{c}_t') - \bm\lambda(\bm{c}_t'')))_s \vert \\
&\leq
\gamma \varepsilon \bm{1}_{\mathcal{S}}^T \vert \bm\Psi \vert^T \bm{1}_d \\
&\leq
\gamma \varepsilon D_\psi d
\end{align*}
where we use the notation $\vert \bm\Psi \vert$ for the matrix that takes element-wise absolute value of $\bm\Psi$.
The second inequality follows since $\Vert \bm\psi(s, \pi) \Vert_2 = \Vert \sum_a \pi(a | s) \bm\varphi(s, a) \Vert_2 \leq \sum_a \pi(a | s) \Vert \bm\varphi(s, a) \Vert_2 \leq \sum_a \pi(a | s) = 1$.
The last inequality follows by the boundedness assumption on $\bm\Phi$.
Hence, choosing $\varepsilon = 1 / \sqrt{dn}$, Term $(a)$ can be bounded by
\begin{align*}
\langle \bm\zeta_t - \bm\zeta_{\bm{w}_t}^{\pi_t}, \bm\Phi^T \bm\mu_{\bm\lambda_t, \pi_t} - \bm\Phi^T \bm\mu_{\bm\lambda_t', \pi_t} \rangle
&\leq
\Vert \bm\zeta_t - \bm\zeta_{\bm{w}_t}^{\pi_t} \Vert_2
\Vert \bm\Phi^T \bm\mu_{\bm\lambda_t, \pi_t} - \bm\Phi^T \bm\mu_{\bm\lambda_t', \pi_t} \Vert_2 \\
&\leq
\frac{2 \gamma D_\zeta D_\psi \sqrt{d}}{\sqrt{n}}.
\end{align*}

\paragraph{Bounding $(b)$}
The second term can be bounded by a union bound of the concentration inequality in Lemma~\ref{lemma:phiTmu-pi} over a $(1 / \sqrt{dn})$-cover of $\mathcal{C}_n'(C^\ast) = \{ \bm{c}' \in [-2C^\ast, 2C^\ast]^n : c_j' = 0 ~\text{if}~ j \in \mathcal{I} \}$, which gives
\begin{align*}
\langle \bm\zeta_t - \bm\zeta^{\pi_t}, \bm\Phi^T \bm\mu_{\bm\lambda_t', \pi_t} - \bm\Phi^T \widehat{\bm\mu}_{\bm\lambda_t', \pi_t} \rangle
&\leq
\Vert \bm\zeta_t - \bm\zeta^{\pi_t} \Vert_2
\Vert \bm\Phi^T \bm\mu_{\bm\lambda_t', \pi_t} - \bm\Phi^T \widehat{\bm\mu}_{\bm\lambda_t', \pi_t} \Vert_2 \\
&\leq
\mathcal{O}\left(
D_\zeta C^\ast \sqrt{\frac{d \log (D_\pi d n / \delta)}{n}}
\right)
\end{align*}

\paragraph{Bounding $(c)$}
Recall from \eqref{eqn:phiTmuhat} that
$\bm\Phi^T \widehat{\bm\mu}_{\bm\lambda, \pi} = (1 - \gamma) \bm\varphi(s_0, \pi)+ \frac{\gamma}{n} \sum_{k = 1}^n c_k \bm\varphi(s_k', \pi)$.
Since $\Vert c_t' - c_t'' \Vert_\infty \leq 1 / \sqrt{dn}$, we have
$$
\Vert \bm\Phi^T \widehat{\bm\mu}_{\bm\lambda(\bm{c}_t''), \pi_t} - \bm\Phi^T \widehat{\bm\mu}_{\bm\lambda(\bm{c}_t'), \pi_t} \Vert_2
=
\frac{\gamma}{n} \left\Vert \sum_{k = 1}^n (c_{tk}'' - c_{tk}') \bm\varphi(s_k', \pi_t) \right\Vert_2
\leq
\gamma / \sqrt{dn}.
$$
It follows by Cauchy-Schwartz that
$$
\langle \bm\zeta_t - \bm\zeta^{\pi_t}, \bm\Phi^T \widehat{\bm\mu}_{\bm\lambda(\bm{c}_t''), \pi_t} - \bm\Phi^T \widehat{\bm\mu}_{\bm\lambda(\bm{c}_t'), \pi_t} \rangle
\leq
\mathcal{O}(D_\zeta / \sqrt{dn}).
$$

\paragraph{Bounding $(d)$}
Recall that $\zeta$-player chooses $\bm\zeta_t \in \mathbb{B}_d(D_\zeta)$ greedily that minimizes $\langle \cdot, \bm\Phi^T \widehat{\bm\mu}_{\bm\lambda(\bm{c}_t'), \pi_t} - \bm\lambda(\bm{c}_t') \rangle$ and that $\bm\zeta_{\bm{w}_t}^{\pi_t} \in \mathbb{B}_d(D_\zeta)$.
Hence, the term $(d)$ can be bounded by
$$
\langle \bm\zeta_t - \bm\zeta^{\pi_t}, \bm\Phi^T \widehat{\bm\mu}_{\bm\lambda(\bm{c}_t'), \pi_t} - \bm\lambda(\bm{c}_t') \rangle \leq 0.
$$

Combining all the bounds, and using $D_\zeta \coloneqq \sqrt{d} + \frac{\gamma D_\psi \sqrt{d}}{1 - \gamma} \leq \mathcal{O}(\frac{\sqrt{d}}{1 - \gamma})$ and $D_\pi = \alpha T D_\zeta \leq \mathcal{O}(\sqrt{\log \vert \mathcal{A} \vert T})$, we get
$$
\textsc{Reg}_t^\zeta \leq \mathcal{O}\left(\frac{C^\ast d}{1 - \gamma} \sqrt{\frac{\log (dn T (\log \vert \mathcal{A} \vert) / \delta)}{n}}\right).
$$
\end{proof}

\subsection{Bounding the Regret of $\lambda$-Player} \label{appendix:regret-lambda}

\begin{lemma} \label{lemma:lambda-regret}
The sequences $\{ \pi_t \}, \{ \bm\lambda(\bm{c}_t) \}, \{\bm\zeta_t \}$ produced by Algorithm~\ref{alg:pdapc} satisfies
$$
\frac{1}{T} \sum_{t = 1}^T \textsc{Reg}^\lambda_t \leq \mathcal{O}\left(\frac{C^\ast d^{3/ 2}}{1 - \gamma} \sqrt{\frac{\log(dnT (\log \vert \mathcal{A} \vert) / \delta)}{n}} \right) + \epsilon_{\text{opt}}^\lambda(T)
$$
with probability at least $1 - \delta$.
\end{lemma}
\begin{proof}
The regret of $\lambda$-player at step $t$ can be bounded by
\begin{align*}
\textsc{Reg}_t^\lambda
&= f(\bm\zeta_t, \bm\lambda^\ast, \pi_t) - f(\bm\zeta_t, \bm\lambda(\bm{c}_t), \pi_t) \\
&=
\langle \bm\lambda^\ast - \bm\lambda(\bm{c}_t), \bm\theta + \gamma \bm\Psi \bm{v}_{\bm\zeta_t, \pi_t} - \bm\zeta_t \rangle \\
&=
\langle \widehat{\bm\lambda}^\ast - \bm\lambda(\bm{c}_t), \bm\theta + \gamma \widehat{\bm\Psi \bm{v}}_{\bm\zeta_t, \pi_t} - \bm\zeta_t \rangle
+ \gamma \langle \widehat{\bm\lambda}^\ast - \bm\lambda(\bm{c}_t), \bm\Psi \bm{v}_{\bm\zeta_t, \pi_t} - \widehat{\bm\Psi \bm{v}}_{\bm\zeta_t, \pi_t} \rangle
+ \langle \bm\lambda^\ast - \widehat{\bm\lambda}^\ast, \bm\theta + \gamma \bm\Psi \bm{v}_{\bm\zeta_t, \pi_t} - \bm\zeta_t \rangle.
\end{align*}
The average of the first term over $t = 1, \dots, T$ is $\epsilon_{\text{opt}}^\lambda(T)$ which vanishes as $T$ increases since the $\lambda$-player employs a no-regret online convex optimization oracle (Definition~\ref{def:oco}) on the sequence of functions $\langle \cdot, \bm\theta + \gamma \widehat{\bm\Psi \bm{v}}_{\bm\zeta_t, \pi_t} - \bm\zeta_t \rangle$.
The second term can be bounded as follows.
\begin{align*}
\langle \widehat{\bm\lambda}^\ast - \bm\lambda(\bm{c}_t), \bm\Psi \bm{v}_{\bm\zeta_t, \pi_t} - \widehat{\bm\Psi \bm{v}}_{\bm\zeta_t, \pi_t} \rangle
&\leq
\Vert \widehat{\bm\lambda}^\ast - \bm\lambda(\bm{c}_t) \Vert_{(n \widehat{\bm\Lambda}_n + \bm{I})^{-1}}
\Vert \bm\Psi \bm{v}_{\bm\zeta_t, \pi_t} - \widehat{\bm\Psi \bm{v}}_{\bm\zeta_t, \pi_t} \Vert_{n \widehat{\bm\Lambda}_n + \bm{I}} \\
&\leq
\frac{1}{\sqrt{n}} \Vert \widehat{\bm\lambda}^\ast - \bm\lambda(\bm{c}_t) \Vert_{\widehat{\bm\Lambda}_n^\dagger}
\Vert \bm\Psi \bm{v}_{\bm\zeta_t, \pi_t} - \widehat{\bm\Psi \bm{v}}_{\bm\zeta_t, \pi_t} \Vert_{n \widehat{\bm\Lambda}_n + \bm{I}} \\
&\leq
\frac{2}{\sqrt{n}} C^\ast \sqrt{d} \cdot \mathcal{O}(D_v \sqrt{d \log (D_\zeta D_\pi n / \delta)} ) \\
&\leq
\mathcal{O}\left(\frac{C^\ast d^{3/ 2}}{1 - \gamma} \sqrt{\frac{\log(dn T (\log \vert \mathcal{A} \vert) / \delta)}{n}} \right)
\end{align*}
where the second inequality follows since $n \widehat{\bm\Lambda}_n + \bm{I} \succcurlyeq n \widehat{\bm\Lambda}_n$ and the fact that both $\widehat{\bm\lambda}^\ast$ and $\bm\lambda(\bm{c}_t)$ are in the column space of $\widehat{\bm\Lambda}$; the third inequality follows by Lemma~\ref{lemma:confidence1} and Lemma~\ref{lemma:ls-uniform} and the fact that the range of $\bm{v}_{\bm\zeta_t, \pi_t}$ is $[0, D_\zeta]$ so that we can set $D_v = D_\zeta$;
the last inequality follows by $D_\zeta \leq \mathcal{O}(\frac{\sqrt{d}}{1 - \gamma})$ and $D_\pi = \alpha T D_\zeta = \mathcal{O}(\sqrt{T \log \vert \mathcal{A} \vert})$.
The third term can be bounded by
\begin{align*}
\langle \bm\lambda^\ast - \widehat{\bm\lambda}^\ast, \bm\theta + \gamma \bm\Psi \bm{v}_{\bm\zeta_t, \pi_t} - \bm\zeta_t \rangle
&\leq
\Vert \bm\lambda^\ast - \widehat{\bm\lambda}^\ast \Vert_2
\Vert \bm\theta + \gamma \bm\Psi \bm{v}_{\bm\zeta_t, \pi_t} - \bm\zeta_t \Vert_2 \\
&\leq
\mathcal{O}\left( C^\ast \sqrt{\frac{\log(d / \delta)}{n}} \right) \cdot \mathcal{O}(D_\zeta \sqrt{d}) \\
&\leq
\mathcal{O}\left(\frac{C^\ast d}{1 - \gamma} \sqrt{\frac{\log(d / \delta)}{n}}\right)
\end{align*}
where the second inequality follows by Lemma~\ref{lemma:lambda-hat} and the last inequality follows by the bound $D_\zeta \leq \mathcal{O}(\frac{\sqrt{d}}{1 - \gamma})$ and the boundedness assumption on $\bm\Psi$.
Combining the three bounds completes the proof.
\end{proof}

\section{Details in Offline Constrained RL Setting} \label{appendix:constrained}

\subsection{Lagrangian Formulation}
Recall that in the linear CMDP setting, the optimization problem of interest is
$$
\begin{aligned}
\max_{\pi} \quad &J_0(\pi) \\
\text{subject to} \quad &J_i(\pi) \geq \tau_i,~~ i = 1, \dots, I.
\end{aligned}
$$
which we denote by $\mathcal{P}(\bm\tau)$ parameterized by the thresholds $\bm\tau \in \mathbb{R}^I$.
We write the Lagrangian function corresponding to the optimization problem $\mathcal{P}(\bm\tau)$ as
$$
L(\pi, \bm{w}) \coloneqq J(\pi) + \bm{w} \cdot (\bm{J}(\pi) - \bm\tau)
$$
where $\bm{J}(\cdot) = (J_1(\cdot), \dots, J_I(\cdot))$, $\bm\tau = (\tau_1, \dots, \tau_I)$ and $\bm{w} \in \mathbb{R}^I$ is the Lagrangian multipliers corresponding to the constraints.
The linear programming formulation of the constrained reinforcement learning problem is:
$$
\begin{aligned}
\max_{\bm\mu \geq \bm{0}} \quad\quad &\langle \bm{r}_0, \bm{\mu} \rangle \\
\text{subject to} \quad\quad
& \langle \bm{r}_i, \bm\mu \rangle \geq \tau_i, \quad i = 1, \dots, I, \\
&\bm{E}^T \bm{\mu} = (1 - \gamma) \bm{\nu}_0 + \gamma \bm{P}^T \bm\mu.
\end{aligned}
$$

Using $\bm{r}_i = \bm\Phi \bm\theta_i$, $i = 0, \dots, I$, and $\bm{P} = \bm\Phi \bm\Psi$, which holds by the linear CMDP assumption (Assumption~\ref{assumption:linear-cmdp}), the linear program can be written as
$$
\begin{aligned}
\max_{\bm\mu \geq 0} \quad\quad &\langle \bm\theta_0, \bm\Phi^T \bm{\mu} \rangle \\
\text{subject to} \quad\quad
&\langle \bm\theta_i, \bm\Phi^T \bm\mu \rangle \geq \tau_i, \quad i = 1, \dots, I, \\
&\bm{E}^T \bm{\mu} = (1 - \gamma) \bm{\nu}_0 + \gamma \bm\Psi^T \bm\Phi^T \bm\mu
\end{aligned}
$$
Note that the optimization variable $\bm\mu \in \mathbb{R}^{\vert \mathcal{S} \times \mathcal{A} \vert}$ is high-dimensional that depends on the size of $\mathcal{S}$.
With the goal of computational and statistical efficiency, we introduce a low-dimensional optimization variable $\bm\lambda = \bm\Phi^T \bm\mu \in \mathbb{R}^d$, which has the interpretation of the average occupancy in the feature space.
With the reparametrization, the optimization problem becomes
$$
\begin{aligned}
\max_{\bm\mu \geq \bm{0}, \bm\lambda} \quad\quad &\langle \bm{\theta}_0, \bm{\lambda} \rangle \\
\text{subject to} \quad\quad
&\langle \bm\theta_i, \bm\lambda \rangle \geq \tau_i, \quad i = 1, \dots, I, \\
&\bm{E}^T \bm{\mu} = (1 - \gamma) \bm{\nu}_0 + \gamma \bm\Psi^T \bm\lambda \\
&\bm\lambda = \bm\Phi^T \bm\mu.
\end{aligned}
$$
The dual of the linear program above is
$$
\begin{aligned}
\min_{\bm{w} \geq \bm{0}, \bm{v}, \bm\zeta} \quad\quad
& (1 - \gamma) \langle \bm\nu_0, \bm{v} \rangle - \langle \bm{w}, \bm\tau \rangle \\
\text{subject to} \quad\quad
& \bm\zeta = \bm\theta_0 + \bm\Theta \bm{w} +  \gamma \bm\Psi \bm{v} \\
& \bm{E} \bm{v} \geq \bm\Phi \bm\zeta.
\end{aligned}
$$
where we write $\bm\Theta = \begin{bmatrix} \bm\theta_1 & \cdots & \bm\theta_I \end{bmatrix} \in \mathbb{R}^{d \times I}$.
The Lagrangian associated to this pair of linear programs is
\begin{align*}
L(\bm\lambda, \bm\mu; \bm{v}, \bm{w}, \bm\zeta)
&= (1 - \gamma)\langle \bm\nu_0, \bm{v} \rangle + \langle \bm\lambda, \bm\theta_0 + \gamma \bm\Psi \bm{v} - \bm\zeta \rangle + \langle \bm\mu, \bm\Phi \bm\zeta - \bm{E} \bm{v} \rangle - \langle \bm{w}, \bm\tau - \bm\Theta^T \bm\lambda \rangle \\
&= \langle \bm\lambda, \bm\theta_0 \rangle + \langle \bm{v}, (1 - \gamma) \bm\nu_0 + \gamma \bm\Psi^T \bm\lambda - \bm{E}^T \bm\mu \rangle + \langle \bm\zeta, \bm\Phi^T \bm\mu - \bm\lambda \rangle - \langle \bm{w}, \bm\tau - \bm\Theta^T \bm\lambda \rangle.
\end{align*}

Note that the optimization variables $\bm\lambda, \bm\zeta \in \mathbb{R}^d$ and $\bm{w} \in \mathbb{R}^I$ are low-dimensional, but $\bm\mu \in \mathbb{R}^{\vert \mathcal{S} \times \mathcal{A} \vert}$ and $\bm{v} \in \mathbb{R}^{\vert \mathcal{S} \vert}$ are not.
With the goal of running a primal-dual algorithm on the Lagrangian using only low-dimensional variables, we introduce policy variable $\pi$ and parameterize $\bm\mu$ and $\bm{v}$, as was done for the unconstrained RL setting, by
\begin{align*}
\mu_{\bm\lambda, \pi}(s, a) &= \pi(a | s) \left[
  (1 - \gamma) \nu_0(s) + \gamma \langle \psi(s), \bm\lambda \rangle
\right] \\
v_{\bm\zeta, \pi}(s) &= \sum_a \pi(a | s) \langle \bm\zeta, \bm\varphi(s, a) \rangle.
\end{align*}
Note that the choice of $\bm\mu_{\bm\lambda, \pi}$ makes the Bellman flow constraint $\bm{E}^T \bm\mu = (1 - \gamma) \bm\nu_0 + \gamma \bm\Psi^T \bm\lambda$ of the primal problem satisfied.
Also, the choice of $\bm{v}_{\bm\zeta, \pi}$ makes $\langle \bm\mu_{\bm\lambda, \pi}, \bm\Phi \bm\zeta - \bm{E} \bm{v}_{\bm\zeta, \pi} \rangle = 0$.
Using the above parameterization, the Lagrangian can be rewritten in terms of $\bm\zeta, \bm\lambda, \bm{w}, \pi$ as follows:
\begin{align}
g(\bm\lambda, \bm\zeta, \bm{w}, \pi)
&= \langle \bm\lambda, \bm\theta_0 \rangle + \langle \bm\zeta, \bm\Phi^T \bm\mu_{\bm\lambda, \pi} - \bm\lambda \rangle - \langle \bm{w}, \bm\tau - \bm\Theta^T \bm\lambda \rangle \label{eqn:lagrangian1} \\
&= (1 - \gamma) \langle \bm\nu_0, \bm{v}_{\bm\zeta, \pi} \rangle + \langle \bm\lambda, \bm\theta_0 + \gamma \bm\Psi \bm{v}_{\bm\zeta, \pi} - \bm\zeta \rangle - \langle \bm{w}, \bm\tau - \bm\Theta^T \bm\lambda \rangle. \label{eqn:lagrangian2}
\end{align}

At the cost of having to keep track of $\pi$, we can now run a primal-dual algorithm on the low-dimensional variables $\bm\zeta$, $\bm\lambda$ and $\bm{w}$.
As is the case for the unconstrained RL setting, the introduction of $\pi$ in the equation does not make the algorithm inefficient because we can only keep track of the distribution $\pi(s | a)$ for state-action pairs that appear in the dataset.

\subsection{Technical Lemmas on Lagrangian} \label{appendix:lagrangian}

For a linearized reward function $u = r_0 + \bm{w} \cdot \bm{r}$ where we use the notation $\bm{r}$ to denote the vector of reward functions $r_1, \dots, r_I$ such that $u(s, a) = r_0(s, a) + \sum_{i = 1}^I w_i r_i(s, a)$, the Bellman equation becomes
\begin{equation} \label{eqn:bellman-factorized}
\bm{Q}_{\bm{w}}^\pi = \bm\Phi (\bm\theta_0 + \bm\Theta \bm{w} + \gamma \bm\Psi \bm{V}_{\bm{w}}^\pi) = \bm\Phi \bm\zeta_{\bm{w}}^\pi
\end{equation}
where we write $\bm{Q}_{\bm{w}}^\pi$ and $\bm{V}_{\bm{w}}^\pi$ as the value functions of the policy $\pi$ with respect to the linearized reward function $r_0 + \bm{w} \cdot \bm{r}$ and define
$$
\bm\zeta_{\bm{w}}^\pi \coloneqq \bm\theta_0 + \bm\Theta \bm{w} + \gamma \bm\Psi \bm{V}_{\bm{w}}^\pi.
$$
Note that if $\bm{w} \in D_w \bm\Delta^I$, we have
$\Vert \bm\zeta_{\bm{w}}^\pi \Vert_2 \leq 1 + D_w + \frac{\gamma \sqrt{d}(1 + D_w)}{1 - \gamma} = \mathcal{O}\left(\frac{D_w \sqrt{d}}{1 - \gamma}\right)$.
We define $D_\zeta \coloneqq 1 + D_w + \frac{\gamma \sqrt{d}(1 + D_w)}{1 - \gamma}$.

\begin{lemma} \label{lemma:lagrangian1}
Let $\bm\zeta_{\bm{w}}^\pi$ be the parameter that satisfies $\bm{Q}_{\bm{w}}^\pi = \bm\Phi \bm\zeta_{\bm{w}}^\pi$ for a given $\bm{w} \in \mathbb{R}^I$ and a policy $\pi$.
Then,
$$
L(\pi, \bm{w}) = g(\bm\zeta_{\bm{w}}^\pi, \bm\lambda, \pi, \bm{w})
$$
for all $\bm\lambda \in \mathbb{R}^d$ in the span of $\{ \bm\varphi(s, a) \}_{(s, a) \in \mathcal{S} \times \mathcal{A}}$.
\end{lemma}
\begin{proof}
For convenience, define the reward function $u(s, a) = r_0(s, a) + \sum_{i = 1}^I w_i r_i(s, a)$.
By the linear CMDP assumption, we have $\bm{u} = \bm\Phi (\bm\theta_0 + \bm\Theta \bm{w})$ where $\bm{u} \in \mathbb{R}^{\vert \mathcal{S} \times \mathcal{A} \vert}$ is the vector representation of the reward function $u$.
Also, by the definition of $v_{\bm\zeta, \pi}$ in \eqref{eqn:v}, we have
$$
v_{\bm\zeta_{\bm{w}}^\pi, \pi}(s) = \sum_a \pi(a | s) \langle \bm\zeta_{\bm{w}}^\pi, \bm\varphi(s, a) \rangle = \sum_a \pi(a | s) Q_{\bm{w}}^\pi(s, a) = V_{\bm{w}}^\pi(s).
$$
Since we assume $\bm\lambda \in \mathbb{R}^d$ is in the span of $\{ \bm\varphi(s, a) \}_{(s, a) \in \mathcal{S} \times \mathcal{A}}$, there exists $\bm\alpha \in \mathbb{R}^{\vert \mathcal{S} \times \mathcal{A} \vert}$ such that $\bm\lambda = \bm\Phi^T \bm\alpha$.
Hence, using the form \eqref{eqn:lagrangian2} of the Lagrangian function, we have
\begin{align*}
g(\bm\zeta_{\bm{w}}^\pi, \bm\lambda, \pi, \bm{w}) 
&=
(1 - \gamma) \langle \bm\nu_0, \bm{v}_{\bm\zeta_{\bm{w}}^\pi} \rangle + \langle \bm\lambda, \bm\theta_0 + \gamma \bm\Psi \bm{v}_{\bm\zeta_{\bm{w}}^\pi} - \bm\zeta_{\bm{w}}^\pi \rangle - \langle \bm{w}, \bm\tau - \bm\Theta^T \bm\lambda \rangle \\
&=
(1 - \gamma) \langle \bm\nu_0, \bm{V}_{\bm{w}}^\pi \rangle + \langle \bm\lambda, \bm\theta_0 + \bm\Theta \bm{w} + \gamma \bm\Psi \bm{V}_{\bm{w}}^\pi - \bm\zeta_{\bm{w}}^\pi \rangle - \langle \bm{w}, \bm\tau \rangle \\
&=
(1 - \gamma) \langle \bm\nu_0, \bm{V}_{\bm{w}}^\pi \rangle + \langle \bm\alpha, \bm\Phi (\bm\theta_0 + \bm\Theta \bm{w} + \gamma \bm\Psi \bm{V}_{\bm{w}}^\pi - \bm\zeta_{\bm{w}}^\pi) \rangle - \langle \bm{w}, \bm\tau \rangle \\
&=
(1 - \gamma) \langle \bm\nu_0, \bm{V}_{\bm{w}}^\pi \rangle + \langle \bm\alpha, \bm{u} + \gamma \bm{P} \bm{V}_{\bm{w}}^\pi - \bm{Q}_{\bm{w}}^\pi \rangle - \langle \bm{w}, \bm\tau \rangle \\
&=
(1 - \gamma) \langle \bm\nu_0, \bm{V}_{\bm{w}}^\pi \rangle - \langle \bm{w}, \bm\tau \rangle \\
&=
L(\pi, \bm{w})
\end{align*}
where the second to last equality uses the Bellman equation \eqref{eqn:bellman} and the last equality is by $L(\pi, \bm{w}) = J_0(\pi) + \bm{w} \cdot (\bm{J}(\pi) - \bm\tau)$ and the fact that $J_0(\pi) + \bm{w} \cdot \bm{J}(\pi)$ is the value of $\pi$ with respect to the linearized value function $r_0 + \bm{w} \cdot \bm{r}$.
\end{proof}
\begin{lemma}
Under the linear MDP setting, let $\bm\zeta^\pi$ be the parameter that satisfies $\bm{Q}^\pi = \bm\Phi \bm\zeta^\pi$ for a policy $\pi$. Then,
$$
J(\pi) = f(\bm\zeta^\pi, \bm\lambda, \pi)
$$
for all $\bm\lambda \in \mathbb{R}^d$ in the span of $\{ \bm\varphi(s, a) \}_{(s, a) \in \mathcal{S} \times \mathcal{A}}$.
\end{lemma}
\begin{proof}
This is a direct corollary of Lemma~\ref{lemma:lagrangian1}, which can be seen by setting $\bm{w} = \bm{0}$.
\end{proof}

\begin{lemma} \label{lemma:lagrangian2}
Given a policy $\pi$, let $\mu^\pi$ be the occupancy measure induced by $\pi$ and let $\bm\lambda^\pi = \bm\Phi^T \bm\mu^\pi$.
Then, for any $\bm\zeta \in \mathbb{R}^d$ and any $\bm{w} \in \mathbb{R}^I$, we have
$$
L(\pi, \bm{w}) = g(\bm\zeta, \bm\lambda^\pi, \bm{w}, \pi).
$$
\end{lemma}
\begin{proof}
By the definition of $\mu_{\bm\lambda, \pi}$ in \eqref{eqn:mu}, we have
\begin{align*}
\mu_{\bm\lambda^\pi, \pi}(s, a) &= \pi(a | s) \left[
  (1 - \gamma) \nu_0(s) + \gamma \langle \bm\psi(s), \bm\lambda^\pi \rangle
\right] \\
&= \pi(a | s) \left[ (1 - \gamma) \nu_0(s) + \gamma \langle \bm\psi(s), \bm\Phi^T \bm\mu^\pi \rangle \right] \\
&= \mu^\pi(s, a).
\end{align*}
Using the form \eqref{eqn:lagrangian1} of the Lagrangian function, we have
\begin{align*}
g(\bm\zeta, \bm\lambda^\pi, \bm{w}, \pi)
&=
\langle \bm\lambda^\pi, \bm\theta_0 \rangle + \langle \bm\zeta, \bm\Phi^T \bm\mu_{\bm\lambda^\pi, \pi} - \bm\lambda^\pi \rangle - \langle \bm{w}, \bm\tau - \bm\Theta^T \bm\lambda^\pi \rangle \\
&=
\langle \bm\Phi^T \bm\mu^\pi, \bm\theta_0 \rangle - \langle \bm{w}, \bm\tau - \bm\Theta^T \bm\Phi^T \bm\mu^\pi \rangle \\
&=
\langle \bm\mu^\pi, \bm{r}_0 \rangle - \langle \bm{w}, \bm\tau - \bm{R}^T \bm\mu^\pi \rangle \\
&=
L(\pi, \bm{w})
\end{align*}
where the second equality uses $\bm\mu_{\bm\lambda^\pi, \pi} = \bm\mu^\pi$ and $\bm\lambda^\pi = \bm\Phi^T \bm\mu^\pi$; and the third equality uses the matrix notation for the reward functions $\bm{R} = \{ r_i(s, a) \}_{(s, a) \in \mathbb{R}^{\vert \mathcal{S} \times \mathcal{A} \vert}, i \in [I]}$; the last equality uses $J_i(\pi) = \langle \bm\mu^\pi, \bm{r}_i \rangle$, $i = 1, \dots, I$.
\end{proof}
\begin{lemma}
Under the linear MDP setting, let $\mu^\pi$ be the occupancy measure induced by a policy $\pi$ and let $\bm\lambda^\pi = \bm\Phi^T \bm\mu^\pi$.
Then, for any $\bm\zeta \in \mathbb{R}^d$, we have
$$
J(\pi) = f(\bm\zeta, \bm\lambda^\pi, \pi).
$$
\end{lemma}
\begin{proof}
This is a direct corollary of Lemma~\ref{lemma:lagrangian2}, which can be seen by setting $\bm{w} = \bm{0}$.
\end{proof}

Define $L_\eta(\pi, \bm{w}) = J_0(\pi) + \bm{w} \cdot (\bm{J}(\pi) - \bm\tau - \eta \bm{1})$ to be the Lagrangian function associated with $\mathcal{P}(\bm\tau + \eta \bm{1})$.
The following lemma shows that the near saddle point of $L_\eta(\cdot, \cdot)$ is a nearly optimal solution of the optimization problem $\mathcal{P}(\bm\tau + \eta \bm{1})$.
\begin{lemma} \label{lemma:key2}
Assume that Slater's condition (Assumption~\ref{assumption:slater}) holds and that $\eta < \phi$ so that $\mathcal{P}(\bm\tau + \eta \bm{1})$ also satisfies Slater's condition.
Suppose $(\bar\pi, \bar{\bm{w}})$ satisfies $L_\eta(\pi, \bar{\bm{w}}) \leq L_\eta(\bar\pi, \bm{w}) + \xi$ for all policies $\pi$ and $\bm{w} \in B \bm\Delta^I$.
Let $(\pi_\eta^\ast, \bm{w}_\eta^\ast)$ be a primal-dual solution to $\mathcal{P}(\bm\tau + \eta \bm{1})$.
Assume $B > \Vert \bm{w}_\eta^\ast \Vert_1$.
Then, we have
\begin{align}
J_0(\bar\pi) &\geq J_0(\pi_\eta^\ast) - \xi \tag{Optimality} \\
J_i(\bar\pi) &\geq \tau_i + \eta - \frac{\xi}{B - \Vert \bm{w}_\eta^\ast \Vert_1},\quad \text{for all}~i = 1, \dots, I \tag{Feasibility}
\end{align}
\end{lemma}
\begin{proof}
We first prove near optimality of $\bar\pi$.

\paragraph{Optimality}

Since $(\bar\pi, \bar{\bm{w}})$ satisfies $L_\eta(\pi, \bar{\bm{w}}) \leq L_\eta(\bar\pi, \bm{w}) + \xi$ for all policies $\pi$ and $\bm{w} \in B \bm\Delta^I$, we have $L_\eta(\pi_\eta^\ast, \bar{\bm{w}}) \leq L_\eta(\bar\pi, \bm{w}) + \xi$ for all $\bm{w} \in B \bm\Delta^I$.
Choosing $\bm{w} = \bm{0}$, we get
$$
L_\eta(\pi_\eta^\ast, \bar{\bm{w}}) \leq L_\eta(\bar\pi, \bm{0}) + \xi = J_0(\bar\pi) + \xi.
$$
Rearranging, we get
$$
J_0(\bar\pi)
\geq
J_0(\pi_\eta^\ast) + \bar{\bm{w}} \cdot (\bm{J}(\pi_\eta^\ast) - \bm\tau - \eta \bm{1}) - \xi
\geq
J_0(\pi_\eta^\ast) - \xi
$$
where the second inequality uses the feasibility of $\pi_\eta^\ast$ for $\mathcal{P}(\bm\tau + \eta \bm{1})$.
Now, we prove feasibility of $\bar\pi$.

\paragraph{Feasibility}

Recall that $(\pi_\eta^\ast, \bm{w}_\eta^\ast)$ is a primal-dual solution to the optimization problem $\mathcal{P}(\bm\tau + \eta \bm{1})$ and $L_\eta(\cdot, \cdot)$ is the Lagrangian function corresponding to the problem $\mathcal{P}(\bm\tau + \eta \bm{1})$.
By strong duality, $(\pi_\eta^\ast, \bm{w}_\eta^\ast)$ is a saddle point for $L_\eta(\cdot, \cdot)$.
Hence, we have
$$
L_\eta(\bar\pi, \bm{w}_\eta^\ast)
\leq
L_\eta(\pi_\eta^\ast, \bm{w}_\eta^\ast)
=
J_0(\pi_\eta^\ast) + \bm{w}_\eta^\ast \cdot (\bm{J}(\pi_\eta^\ast) - \bm\tau - \eta \bm{1})
=
J_0(\pi_\eta^\ast)
$$
where the first inequality follows from the fact that $(\pi_\eta^\ast, \bm{w}_\eta^\ast)$ is a saddle point of $L_\eta(\cdot, \cdot)$ and the last equality follows from the complementary slackness property of the solution $(\pi_\eta^\ast, \bm{w}_\eta^\ast)$.
Rearranging, we get
\begin{equation} \label{eqn:perf-diff-lower-bound}
J_0(\pi_\eta^\ast) - J_0(\bar\pi)
\geq
\bm{w}_\eta^\ast \cdot (\bm{J}(\bar\pi) - \bm\tau - \eta \bm{1})
\geq
(m - \eta) \Vert \bm{w}_\eta^\ast \Vert_1
\end{equation}
where we define $m = \min_{i \in [I]} (J_i(\bar\pi) - \tau_i)$.
Now, to upper bound $J_0(\pi_\eta^\ast) - J_0(\bar\pi)$, we first use the feasibility of $\pi_\eta^\ast$ for $\mathcal{P}(\bm\tau + \eta \bm{1})$ as follows.
$$
L_\eta(\pi_\eta^\ast, \bar{\bm{w}})
=
J_0(\pi_\eta^\ast) + \bar{\bm{w}} \cdot (\bm{J}(\pi_\eta^\ast) - \bm\tau - \eta \bm{1})
\geq
J_0(\pi_\eta^\ast).
$$
On the other hand, since $(\bar\pi, \bar{\bm{w}})$ satisfies $L(\pi, \bar{\bm{w}}) \leq L(\bar\pi, \bm{w}) + \xi$ for all policies $\pi$ and $\bm{w} \in B \bm\Delta^I$, we have $L_\eta(\pi_\eta^\ast, \bar{\bm{w}}) \leq L_\eta(\bar\pi, \bm{w}) + \xi$ for any $\bm{w} \in B \bm\Delta^I$.
By choosing $\bm{w}$ such that $w_j = B$ for $j = \argmin_{i \in [I]} (J_i(\bar\pi) - \tau_i)$ and recalling $m = \min_{i \in [I]} (J_i(\bar\pi) - \tau_i)$, we get
$$
L_\eta(\pi_\eta^\ast, \bar{\bm{w}})
\leq
L_\eta(\bar\pi, \bm{w}) + \xi
= J_0(\bar\pi) + B (m - \eta) + \xi.
$$
Combining the previous two results (upper bound and lower bound of $L_\eta(\pi_\eta^\ast, \bar{\bm{w}})$), we get
\begin{equation} \label{eqn:perf-diff-upper-bound}
J_0(\pi_\eta^\ast) - J_0(\bar\pi) \leq B (m - \eta) + \xi.
\end{equation}
Combining the lower bound (\ref{eqn:perf-diff-lower-bound}) and the upper bound (\ref{eqn:perf-diff-upper-bound}) of $J_0(\pi_\eta^\ast) - J_0(\bar\pi)$ and rearranging, we get
$$
m - \eta \geq \frac{- \xi}{B - \Vert \bm{w}_\eta^\ast \Vert_1}.
$$
Since $J_i(\bar\pi) - \tau_i - \eta \geq m - \eta$ for all $i \in [I]$, rearranging the above gives
$$
J_i(\bar\pi) \geq \tau_i + \eta - \frac{\xi}{B - \Vert \bm{w}_\eta^\ast \Vert_1}
$$
for all $i = 1, \dots, I$.
\end{proof}

\begin{lemma}[Lemma 13 in \textcite{hong2023primal}] \label{lemma:dual-variable-bound}
Consider a constrained optimization problem $\mathcal{P}(\bm\tau)$ with threshold $\bm\tau = (\tau_1, \dots, \tau_I)$ with $\tau_i > 0$ for all $i = 1, \dots, I$.
Suppose the problem satisfies Slater's condition with margin $\varphi > 0$, in other words, there exists $\pi \in \Pi$ that satisfies the constraint $J_i(\pi) \geq \tau_i + \phi$ for all $i = 1, \dots, I$.
Then, the optimal dual variable $\bm\lambda^\ast$ of the problem satisfies $\Vert \bm\lambda^\ast \Vert_1 \leq \frac{1}{\varphi}$.
\end{lemma}
\begin{proof}
Let $\pi^\ast$ be an optimal policy of the optimization problem $\mathcal{P}(\bm\tau)$.
Define the dual function $f(\bm\lambda) = \max_\pi J_0(\pi) + \bm\lambda \cdot (\bm{J}(\pi) - \bm\tau)$.
Let $\bm\lambda^\ast = \argmin_{\bm\lambda \in \mathbb{R}_+^I} f(\bm\lambda)$.
Trivially, $\lambda_i^\ast \geq 0$ for all $i = 1, \dots, I$.
Also, by strong duality, we have $f(\bm\lambda^\ast) = J_0(\pi^\ast)$.
Let $\widehat{\pi}$ be a feasible policy with $\bm{J}(\widehat{\pi}) \geq \bm\tau + \phi \bm{1}$ where the inequality is component-wise and $\bm{1} = (1, \dots, 1)$.
Such a policy exists by the assumption of this lemma.
Then,
$$
J_0(\pi^\ast)
= f(\bm\lambda^\ast)
\geq J_0(\widehat\pi) + \bm\lambda^\ast \cdot (\bm{J}(\widehat\pi) - \bm\tau)
\geq J_0(\widehat\pi) + \bm\lambda^\ast \cdot \phi \bm{1} = J_0(\widehat\pi) + \phi \Vert \bm\lambda^\ast \Vert_1.
$$
Rearranging and using $1 \geq J_0(\pi^\ast) \geq J_0(\widehat\pi) \geq 0$ completes the proof:
$$
\Vert \bm\lambda^\ast \Vert_1 \leq \frac{J_0(\pi^\ast) - J_0(\widehat\pi)}{\phi} \leq \frac{1}{\varphi}.
$$
\end{proof}

\subsection{Proof of Theorem~\ref{theorem:main}}
For a given $\bm{w} \in \mathbb{R}^I$ and a policy $\pi$, define $\bm\zeta_{\bm{w}}^\pi \in \mathbb{R}^d$ to be the parameter that satisfies $\bm{Q}_{\bm{w}}^\pi = \bm\Phi \bm\zeta_{\bm{w}}^\pi$ where $\bm{Q}_{\bm{w}}^\pi$ is the state-action value function of the policy $\pi$ with respect to the reward function $r_0 + \bm{w} \cdot \bm{r}$.
Using $g(\bm\zeta_{\bm{w}}^\pi, \bm\lambda, \bm{w}, \pi) = L(\pi, \bm{w})$ for any $\bm\lambda$ that is a linear combination of $\{ \bm\varphi(s, a) \}_{(s, a) \in \mathcal{S} \times \mathcal{A}}$ (Lemma~\ref{eqn:lagrangian1}) and $g(\bm\zeta, \bm\lambda^\pi, \bm{w}, \pi) = L(\pi, \bm{w})$ for any $\bm\zeta \in \mathbb{R}^d$ where $\bm\lambda^\pi = \bm\Phi^T \bm\mu^\pi$ (Lemma~\ref{eqn:lagrangian2}), we have
\begin{align*}
L(\pi^\ast, \bm{w}_t) - L(\pi_t, \bm{w})
&= g(\bm\zeta_t, \bm\lambda^\ast, \bm{w}_t, \pi)
 - g(\bm\zeta_{\bm{w}_t}^{\pi_t}, \bm\lambda(\bm{c}'_t), \bm{w}, \pi_t)) \\
&=
(\underbrace{g(\bm\zeta_t, \bm\lambda^\ast, \bm{w}_t, \pi^\ast) - g(\bm\zeta_t, \bm\lambda^\ast, \bm{w}_t, \pi_t)}_{\textsc{Reg}^\pi_t}) \\
&\hspace{10mm}+(\underbrace{g(\bm\zeta_t, \bm\lambda^\ast, \bm{w}_t, \pi_t) - g(\bm\zeta_t, \bm\lambda(\bm{c}'_t), \bm{w}_t, \pi_t)}_{\textsc{Reg}^\lambda_t}) \\
&\hspace{10mm}+
(\underbrace{g(\bm\zeta_t, \bm\lambda(\bm{c}'_t), \bm{w}_t, \pi_t) - g(\bm\zeta_t, \bm\lambda(\bm{c}'_t), \bm{w}, \pi_t)}_{\textsc{Reg}^w_t}) \\
&\hspace{10mm}+
(\underbrace{g(\bm\zeta_t, \bm\lambda(\bm{c}'_t), \bm{w}, \pi_t) - g(\bm\zeta_{\bm{w}_t}^{\pi_t}, \bm\lambda(\bm{c}'_t), \bm{w}, \pi_t)}_{\textsc{Reg}^\zeta_t})
\end{align*}
where $\pi^\ast$ is an optimal policy for the optimization problem $\mathcal{P}(\bm\tau)$ and we use the notation $\bm\lambda^\ast = \bm\lambda^{\pi^\ast}$.
Note that the suboptimality $L(\pi^\ast, \bm{w}_t) - L(\pi_t, \bm{w})$ is decomposed into regret terms of the four players.
As long as we show that the sum of the four regrets over $t = 1, \dots, T$ are sublinear in $T$ and the dataset size $n$, we obtain
$\frac{1}{T} \sum_{t = 1}^T L(\pi^\ast, \bm{w}_t) - L(\pi_t, \bm{w}) = L(\pi^\ast, \bar{\bm{w}}) - L(\bar{\pi}, \bm{w}) = o(1)$ where $\bar{\bm{w}} = \frac{1}{T} \sum_{t = 1}^T \bm{w}_t$ and $\bar{\pi} = \text{Unif}(\pi_1, \dots, \pi_T)$ is the mixture policy that chooses a policy among $\pi_1, \dots, \pi_T$ uniformly at random and runs the chosen policy for the entire trajectory.
Then, for large enough $T$ and $n$, we get $L(\pi^\ast, \bar{\bm{w}}) \leq L(\bar\pi, \bm{w}) + \epsilon$ where $\epsilon$ vanishes as $T$ and $n$ increase.
Such a pair is a near saddle point of the Lagrangian function $L(\cdot, \cdot)$ and it can be shown that the mixture policy $\bar\pi$ is a near-optimal solution of the optimization problem \eqref{eqn:opt}.
Specifically, adapting the proof of \textcite{hong2023primal}, we can  show that if the Slater's condition (Assumption~\ref{assumption:slater}) holds, then a near saddle point $(\bar\pi, \bar{\bm{w}})$ of $L(\cdot, \cdot)$ with $L(\pi, \bar{\bm{w}}) \leq L(\bar\pi, \bm{w}) + \mathcal{O}(\epsilon)$ for all policies $\pi$ and $\bm{w} \in (1 + \frac{1}{\phi}) \bm\Delta^I$ satisfies
\begin{align*}
J_0(\bar\pi) &\geq J_0(\pi^\ast) - \epsilon \\
J_i(\bar\pi) &\geq \tau_i - \epsilon, \quad i = 1, \dots, I
\end{align*}
where $\pi^\ast$ is the optimal policy for $\mathcal{P}(\bm\tau)$, implying that $\bar\pi$ is a nearly optimal solution for the optimization problem.

In the rest of the section, we sketch the analysis that shows that $L(\pi, \bar{\bm{w}}) - L(\bar\pi, \bm{w}) \leq \epsilon$ for large enough $T$ and $n = \mathcal{O}(\epsilon^{-2})$.
With the decomposition of $L(\pi, \bm{w}_t) - L(\pi_t, \bm{w})$ into regrets of the four players discussed previously, we study how the four regrets can be bounded in the next four subsections.

\subsubsection{Bounding Regret of $\pi$-Player}

Using the expression \eqref{eqn:lagrangian2}, the regret of $\pi$-player simplifies to
\begin{align*}
\text{Reg}^\pi_t &=
g(\bm\zeta_t, \bm\lambda^\ast, \bm{w}_t, \pi^\ast) - g(\bm\zeta_t, \bm\lambda^\ast, \bm{w}_t, \pi_t) \\
&= \langle \bm\nu^\ast, \bm{v}_{\bm\zeta_t, \pi} - \bm{v}_{\bm\zeta_t, \pi_t} \rangle \\
&= \langle \bm\nu^\ast, \textstyle \sum_a (\pi(a | \cdot ) - \pi_t(a | \cdot)) \langle \bm\zeta_t, \bm\varphi(\cdot, a) \rangle \rangle.
\end{align*}
where $\bm\nu^\ast = (1 - \gamma) \bm\nu_0 + \gamma \bm\Psi^T \bm\lambda^\ast$ is the state occupancy measure induced by $\pi^\ast$.
This is identical to the regret term for the $\pi$-player in the unconstrained RL setting.
As is done in the unconstrained RL setting, choosing $\alpha = \mathcal{O}((1 - \gamma) \sqrt{\log \vert \mathcal{A} \vert / (dT)})$, we get
$$
\frac{1}{T} \sum_{t = 1}^T \textsc{Reg}_t^\pi \leq \mathcal{O}\left(\frac{1}{1 - \gamma}\sqrt{(d \log \vert \mathcal{A} \vert) / T }\right)
$$
which is sublinear in $T$.
Consequently, choosing $T$ to be at least $\Omega(\frac{d \log \vert \mathcal{A} \vert}{(1 - \gamma)^2 \epsilon^2})$ gives $\frac{1}{T} \sum_{t = 1}^T \text{Reg}_t^\pi \leq \epsilon$.

\subsubsection{Bounding Regret of $\zeta$-Player}
Note that the regret for the $\zeta$-player simplifies to
\begin{align*}
\textsc{Reg}_t^\zeta &= g(\bm\zeta_t, \bm\lambda(\bm{c}'_t), \bm{w}, \pi_t) - g(\bm\zeta_{\bm{w}_t}^{\pi_t}, \bm\lambda(\bm{c}'_t), \bm{w}, \pi_t) \\
&=
\langle \bm\zeta_t - \bm\zeta_{\bm{w}_t}^{\pi_t}, \bm\Phi^T \bm\mu_{\bm\lambda_t, \pi_t} - \bm\lambda_t \rangle
\end{align*}
which has the same form as in the unconstrained case.
The proof is essentially the same as the proof in Section~\ref{appendix:zeta-regret} for the unconstrained setting.
The only difference is that $D_\zeta \leq \mathcal{O}(\frac{D_w \sqrt{d}}{1 - \gamma})$ where $D_w = 1 + \frac{1}{\phi}$.
Following the proof, we get
$$
\textsc{Reg}_t^\zeta \leq \mathcal{O}\left(\frac{C^\ast d}{(1 - \gamma)\phi} \sqrt{\frac{\log (dn T (\log \vert \mathcal{A} \vert) / (\delta \phi))}{n}}\right).
$$

\subsubsection{Bounding Regret of $\lambda$-Player}
Using the expression \eqref{eqn:lagrangian2}, the regret of $\lambda$-player simplifies to
\begin{align*}
\textsc{Reg}_t^\lambda &= f(\bm\zeta_t, \bm\lambda^\ast, \bm{w}_t, \pi_t) - f(\bm\zeta_t, \bm\lambda_t, \bm{w}_t, \pi_t) \\
&=
\langle \bm\lambda^\ast - \bm\lambda_t, \underbrace{\bm\theta_0 + \gamma \bm\Psi \bm{v}_{\bm\zeta_t, \pi_t} - \bm\zeta_t + \bm\Theta \bm{w}_t}_{= \bm\xi_t} \rangle
\end{align*}
Following the analysis for the unconstrained setting in Section~\ref{appendix:regret-lambda}, we get
$$
\frac{1}{T} \sum_{t = 1}^T \textsc{Reg}^\lambda_t \leq \mathcal{O}\left(\frac{C^\ast d^{3/ 2}}{(1 - \gamma)\phi} \sqrt{\frac{\log(dnT (\log \vert \mathcal{A} \vert) / (\delta\phi))}{n}} \right) + \epsilon_{\text{opt}}^\lambda(T)
$$

\subsubsection{Bounding Regret of $w$-Player} \label{section:regret-w}
Using expression \eqref{eqn:lagrangian1}, the regret of $w$-player simplifies to
\begin{align*}
g(\bm\zeta_t, \bm\lambda(\bm{c}_t'), \bm{w}_t, \pi_t) - g(\bm\zeta_t, \bm\lambda(\bm{c}_t'), \bm{w}, \pi_t)
=
\langle \bm{w}_t - \bm{w}, \bm\tau - \bm\Theta^T \bm\lambda_t \rangle
\end{align*}
which is bounded by $0$ since the $w$-player choose $\bm{w}_t \in D_w \Delta^I$ that minimizes $\langle \cdot, \bm\tau - \bm\Theta^T \bm\lambda_t \rangle$.

\subsection{Exact Feasibility} \label{appendix:exact-feasibility}

For producing an $\epsilon$-optimal policy that satisfies the constraints exactly, we make the following two-policy feature coverage assumptions.

\begin{assumption}[Two-policy feature coverage] \label{assumption:two-policy}
Assume the Slater's condition (Assumption~\ref{assumption:slater}) holds.
Denote by $\pi^\ast_\phi$ an optimal policy for the optimization problem $\mathcal{P}(\bm\tau + \phi \bm{1})$.
Denote by $\pi^\ast$ an optimal policy for $\mathcal{P}(\bm\tau)$.
Assume that
$$
(\bm\lambda^\ast)^T (\bm\Lambda^\dagger)^2 \bm\lambda^\ast \leq C^\ast, \quad\quad
(\bm\lambda_\phi^\ast)^T (\bm\Lambda^\dagger)^2 \bm\lambda_\phi^\ast \leq C^\ast.
$$
\end{assumption}

With the assumption above and the linear CMDP assumption (Assumption~\ref{assumption:linear-cmdp}), consider running Algorithm~\ref{alg:pdapc} with stricter thresholds $\bm\tau + \eta \bm{1}$ where $\eta = \phi \epsilon$ and $D_w = \frac{4}{\phi}$.
Since the Slater's constant for $\mathcal{P}(\bm\tau + \eta \bm{1})$ is $\phi(1 - \epsilon)$, following the main analysis gives
\begin{align}
L_\eta(\pi^\ast, \bar{\bm{w}}) &\leq L_\eta(\bar{\pi}, \bm{w}) + \epsilon \label{eqn:saddle1} \\
L_\eta(\pi^\ast_\phi, \bar{\bm{w}}) &\leq L_\eta(\bar{\pi}, \bm{w}) + \epsilon \label{eqn:saddle2}
\end{align}
for any $\bm{w} \in D_w \bm\Delta^I$ with probability at least $1 - \delta$ for sample size $n = \widetilde{\mathcal{O}}\left( \frac{(C^\ast)^2 d^3 \log(d / \delta)}{(1 - \gamma)^2 \phi^2 \epsilon^2} \right)$ since changing $\phi$ to $\phi(1 - \epsilon)$ does not affect the order of sample size bound where $L_\eta(\pi, \bm{w}) = J_0(\pi) + \bm{w} \cdot (\bm{J}(\pi) - \bm\tau - \eta \bm{1})$ is the Lagrangian function for $\mathcal{P}(\bm\tau + \eta \bm{1})$.
Then following the proof of Theorem 3 in \textcite{hong2023primal}, we can argue as follows.

\paragraph{Near Optimality}

Setting $\bm{w} = \bm{0}$ in (\ref{eqn:saddle1}) and rearranging, we get
\begin{align*}
J_0(\bar\pi)
&\geq
J_0(\pi^\ast) + \bar{\bm{w}} \cdot (\bm{J}(\pi^\ast) - \bm\tau - \eta \bm{1}) - \epsilon \\
&\geq
J_0(\pi^\ast) - \eta \Vert \bar{\bm{w}} \Vert_1 - \mathcal{O}(\epsilon) \\
&\geq
J_0(\pi^\ast) - \mathcal{O}(\epsilon)
\end{align*}
where the second inequality follows by the feasibility of $\pi^\ast$ for $\mathcal{P}(\bm\tau)$;
the last inequality follows by $\eta \Vert \bar{\bm{w}} \Vert_1 \leq \eta D_w = \mathcal{O}(\epsilon)$.
This proves near optimality of $\bar\pi$.
Now we prove that $\bar\pi$ is (exactly) feasible for $\mathcal{P}(\bm\tau)$.

\paragraph{Exact Feasibility}
Define $m = \min_{i \in [I]} (J_i(\bar\pi) - \tau_i)$.
If $m \geq 0$ then $J_i(\bar\pi) - \tau_i \geq 0$ for all $i = 1, \dots, I$ and exact feasibility trivially holds.
We only consider the case where $m < 0$.
Define a mixture policy $\widetilde\pi = (1 - \zeta) \pi^\ast + \zeta \pi_\phi^\ast$ where $\zeta \in (0, 1)$ is to be determined later.
The mixture policy has the interpretation of first drawing a policy from $\{ \pi^\ast, \pi_\phi^\ast \}$ with probabilities $(1 - \zeta)$ and $\zeta$, then running the drawn policy for the entire trajectory.
Since $L_\eta(\cdot, \bar{\bm{w}})$ is linear, a linear combination of (\ref{eqn:saddle1}) and (\ref{eqn:saddle2}) with coefficients $1 - \zeta$ and $\zeta$ respectively, we get
$$
L_\eta(\widetilde\pi, \bar{\bm{w}}) \leq L_\eta(\bar\pi, \bm{w}) + \epsilon.
$$
Choosing $\bm{w}$ such that $w_j = D_w$ for $j = \argmin_{i \in [I]} (J_i(\bar\pi) - \tau_i)$ and $w_j = 0$ for all other indices, we get
\begin{align*}
L_\eta(\widetilde\pi, \bar{\bm{w}})
&\leq
J_0(\bar\pi) + \bm{w} \cdot (\bm{J}(\bar\pi) - \bm\tau - \eta \bm{1}) + \epsilon \\
&=
J_0(\bar\pi) + D_w(m - \eta) + \epsilon.
\end{align*}
On the other hand, using the fact that $\widetilde\pi$ is feasible for $\mathcal{P}(\bm\tau + \zeta \phi \bm{1})$, we get
\begin{align*}
L_\eta(\widetilde\pi, \bar{\bm{w}})
&=
J_0(\widetilde\pi) + \bar{\bm{w}} \cdot (\bm{J}(\widetilde\pi) - \bm\tau - \eta \bm{1}) \\
&\geq
J_0(\widetilde\pi) + (\zeta \phi - \eta) \Vert \bar{\bm{w}} \Vert_1.
\end{align*}
Combining the previous two results (upper bound and lower bound of $L_\eta(\widetilde\pi, \bar{\bm{w}})$) and rearranging, we get
\begin{equation}
J_0(\widetilde\pi) - J_0(\bar\pi) \leq D_w(m - \eta) - (\zeta \phi - \eta) \Vert \bar{\bm{w}} \Vert_1 + \epsilon. \label{eqn:exact-feasibility-intermediate}
\end{equation}
Now, to get a lower bound of $J_0(\widetilde\pi) - J_0(\bar\pi)$, let $(\widetilde\pi^\ast, \widetilde{\bm{w}}^\ast)$ be a primal-dual solution of $\mathcal{P}(\bm\tau + \zeta \phi \bm{1})$.
Note that $\mathcal{P}(\bm\tau + \zeta \phi \bm{1})$ is feasible by the Slater's condition assumption \ref{assumption:slater} and the fact that $\zeta \phi \in (0, \phi)$.
Since $(\widetilde\pi^\ast, \widetilde{\bm{w}}^\ast)$ is a saddle point of $L_{\zeta \phi}(\pi, \bm{w})
= J_0(\pi) + \bm{w} \cdot (J_{\bm{C}}(\pi) - \bm\tau - \zeta \phi \bm{1})$, we get
$$
L_{\zeta \phi}(\bar\pi, \widetilde{\bm{w}}^\ast) \leq L_{\zeta \phi}(\widetilde\pi^\ast, \widetilde{\bm{w}}^\ast) = J_0(\widetilde\pi^\ast) \leq J_0(\pi^\ast) \leq \frac{1}{1 - \zeta} J_0(\widetilde\pi)
$$
where the equality follows by the complementary slackness property;
the second inequality follows since the feasibility set of $\mathcal{P}(\bm\tau)$ contains that of $\mathcal{P}(\bm\tau + \zeta \phi \bm{1})$;
and the last inequality follows by $J_0(\widetilde\pi) = (1 - \zeta) J_0(\pi^\ast) + \zeta J_0(\pi_\phi^\ast) \geq (1 - \zeta) J_0(\pi^\ast)$.
Rearranging, we get
\begin{align*}
J_0(\widetilde\pi) - J_0(\bar\pi)
&\geq - \zeta J_0(\bar\pi) + (1 - \zeta) \widetilde{\bm{w}}^\ast \cdot (\bm{J}(\bar\pi) - \bm\tau - \zeta \phi \bm{1}) \\
&\geq
-\zeta + (1 - \zeta) (m - \zeta \phi) \Vert \widetilde{\bm{w}}^\ast \Vert_1
\end{align*}
where the second inequality follows by $J_0(\cdot) \leq 1$ and the definition of $m$.
Combining with the upper bound of $J_0(\widetilde\pi) - J_0(\bar\pi)$ shown in (\ref{eqn:exact-feasibility-intermediate}) and rearranging, we get
\begin{equation} \label{eqn:exact-feasibility-intermediate2}
(D_w - (1 - \zeta) \Vert \widetilde{\bm{w}}^\ast \Vert_1)m
\geq
D_w\eta + (\zeta \phi - \eta) \Vert \bar{\bm{w}} \Vert_1 - \zeta - (1 - \zeta) \zeta \phi \Vert \widetilde{\bm{w}}^\ast \Vert_1 - \epsilon.
\end{equation}
Now, we choose our parameters as follows.
$$
\zeta = \epsilon, \quad
D_w = \frac{4}{\phi}, \quad
\eta = \phi \epsilon.
$$
Since $\widetilde{\bm{w}}^\ast$ is a dual solution of $\mathcal{P}(\bm\tau + \zeta \phi \bm{1})$, which has a margin of $\phi - \zeta \phi$, Lemma~\ref{lemma:dual-variable-bound} gives $\Vert \widetilde{\bm{w}}^\ast \Vert_1 \leq \frac{1}{\phi - \zeta \phi}$.
Hence,
$$
\zeta \phi \Vert \widetilde{\bm{w}}^\ast \Vert_1
\leq
\frac{\zeta \phi}{\phi - \zeta \phi}
\leq 2\zeta
= 2\epsilon
$$
where the second inequality uses $\zeta = \epsilon \leq \frac{1}{2}$.
Hence, $\Vert \widetilde{\bm{w}}^\ast \Vert_1 \leq \frac{2 \epsilon}{\zeta \phi} = \frac{2}{\phi} < D_w$ so that $D_w - (1 - \zeta) \Vert \widetilde{\bm{w}}^\ast \Vert_1 > 0$.
Hence, the previous result (\ref{eqn:exact-feasibility-intermediate2}) gives
\begin{align*}
(D_w - (1 - \zeta) \Vert \widetilde{\bm{w}}^\ast \Vert_1)m
&\geq
D_w\eta + (\zeta \phi - \eta) \Vert \bar{\bm\lambda} \Vert_1 - \zeta - (1 - \zeta) \zeta \phi \Vert \widetilde{\bm{w}}^\ast \Vert_1 - \epsilon \\
&\geq
4\epsilon + 0 - \epsilon - 2\epsilon - \epsilon \\
&=
0.
\end{align*}
Since $D_w - (1 - \zeta) \Vert \widetilde{\bm\lambda}^\ast \Vert_1 > 0$, we have $m \geq 0$ which implies $\tau_i - J_i(\bar\pi) \geq 0$ for all $i = 1, \dots, I$.
This leads to the following theorem.
\begin{theorem}
Under assumptions~\ref{assumption:linear-cmdp} and \ref{assumption:two-policy}, as long as $T$ is at least $\Omega(\frac{d \log \vert \mathcal{A} \vert}{(1 - \gamma)^2 \epsilon^2})$, the policy $\bar\pi$ produced by Algorithm~\ref{alg:pdapc} with thresholds $\bm\tau + \phi \epsilon \bm{1}$ and $D_w = \frac{4}{\phi}$ satisfies $J_0(\bar\pi) \geq J_0(\pi^\ast) - \epsilon$ and $J_i(\bar\pi) \geq \tau_i$, $i = 1, \dots, I$ with probability at least $1 - \delta$ for sample size 
$$
n = \mathcal{O}\left(
\frac{(C^\ast)^2 d^3 \log(dn (\log \vert \mathcal{A} \vert) / (\delta \phi \epsilon(1 - \gamma))}{(1 - \gamma)^2 \phi^2 \epsilon^2}
\right).
$$
\end{theorem}

\end{document}